\documentclass[english]{elsarticle}
\usepackage[T1]{fontenc}
\usepackage[latin9]{inputenc}
\usepackage{babel}
\usepackage{float}
\usepackage{amsbsy}
\usepackage{amstext}
\usepackage{amsthm}
\usepackage{graphicx}

\makeatletter

\providecommand{\tabularnewline}{\\}
\floatstyle{ruled}
\newfloat{algorithm}{tbp}{loa}
\providecommand{\algorithmname}{Algorithm}
\floatname{algorithm}{\protect\algorithmname}

\theoremstyle{plain}
\newtheorem{thm}{\protect\theoremname}
\theoremstyle{plain}
\newtheorem{lem}[thm]{\protect\lemmaname}

\usepackage{times}
\usepackage{amsmath, amsthm, amssymb}

\@ifundefined{showcaptionsetup}{}{%
 \PassOptionsToPackage{caption=false}{subfig}}
\usepackage{subfig}
\makeatother

\providecommand{\lemmaname}{Lemma}
\providecommand{\theoremname}{Theorem}

\begin{document}

\begin{frontmatter}{}

\title{Incorporating Expert Prior in Bayesian Optimisation via Space Warping}

\author{Anil Ramachandran\corref{*}}

\ead{aramac@deakin.edu.au}

\author{Sunil Gupta}

\ead{sunil.gupta@deakin.edu.au}

\author{Santu Rana}

\ead{santu.rana@deakin.edu.au}

\author{Cheng Li}

\ead{cheng.l@deakin.edu.au}

\author{Svetha Venkatesh}

\ead{svetha.venkatesh@deakin.edu.au}

\address{Applied Artificial Intelligence Institute (A\textsuperscript{2}I\textsuperscript{2}),
Deakin University, Australia}

\cortext[{*}]{Corresponding Author}
\begin{abstract}
Bayesian optimisation is a well-known sample-efficient method for
the optimisation of expensive black-box functions. However when dealing
with big search spaces the algorithm goes through several low function
value regions before reaching the optimum of the function. Since the
function evaluations are expensive in terms of both money and time,
it may be desirable to alleviate this problem. One approach to subside
this cold start phase is to use prior knowledge that can accelerate
the optimisation. In its standard form, Bayesian optimisation assumes
the likelihood of any point in the search space being the optimum
is equal. Therefore any prior knowledge that can provide information
about the optimum of the function would elevate the optimisation performance.
In this paper, we represent the prior knowledge about the function
optimum through a prior distribution. The prior distribution is then
used to warp the search space in such a way that space gets expanded
around the high probability region of function optimum and shrinks
around low probability region of optimum. We incorporate this prior
directly in function model (Gaussian process), by redefining the kernel
matrix, which allows this method to work with any acquisition function,
i.e. acquisition agnostic approach. We show the superiority of our
method over standard Bayesian optimisation method through optimisation
of several benchmark functions and hyperparameter tuning of two algorithms:
Support Vector Machine (SVM) and Random forest.
\end{abstract}

\end{frontmatter}{}

\global\long\def\mF{\mathcal{F}}

\global\long\def\mA{\mathcal{A}}

\global\long\def\mH{\mathcal{H}}

\global\long\def\mX{\mathcal{X}}

\global\long\def\dist{d}

\global\long\def\HX{\entro\left(X\right)}
 \global\long\def\entropyX{\HX}

\global\long\def\HY{\entro\left(Y\right)}
 \global\long\def\entropyY{\HY}

\global\long\def\HXY{\entro\left(X,Y\right)}
 \global\long\def\entropyXY{\HXY}

\global\long\def\mutualXY{\mutual\left(X;Y\right)}
 \global\long\def\mutinfoXY{W\mutualXY}

\global\long\def\given{\mid}

\global\long\def\gv{\given}

\global\long\def\goto{\rightarrow}

\global\long\def\asgoto{\stackrel{a.s.}{\longrightarrow}}

\global\long\def\pgoto{\stackrel{p}{\longrightarrow}}

\global\long\def\dgoto{\stackrel{d}{\longrightarrow}}

\global\long\def\ll{\mathit{l}}

\global\long\def\logll{\mathcal{L}}

\global\long\def\bzero{\vt0}

\global\long\def\bone{\mathbf{1}}

\global\long\def\bff{\vt f}

\global\long\def\bx{\boldsymbol{x}}

\global\long\def\bX{\boldsymbol{X}}

\global\long\def\bW{\mathbf{W}}

\global\long\def\bH{\mathbf{H}}

\global\long\def\bL{\mathbf{L}}

\global\long\def\tbx{\tilde{\bx}}

\global\long\def\by{\boldsymbol{y}}

\global\long\def\bY{\boldsymbol{Y}}

\global\long\def\bz{\boldsymbol{z}}

\global\long\def\bZ{\boldsymbol{Z}}

\global\long\def\bu{\boldsymbol{u}}

\global\long\def\bU{\boldsymbol{U}}

\global\long\def\bv{\boldsymbol{v}}

\global\long\def\bV{\boldsymbol{V}}

\global\long\def\bw{\vt w}

\global\long\def\balpha{\gvt\alpha}

\global\long\def\bbeta{\gvt\beta}

\global\long\def\bmu{\gvt\mu}

\global\long\def\btheta{\boldsymbol{\theta}}

\global\long\def\blambda{\boldsymbol{\lambda}}

\global\long\def\realset{\mathbb{R}}

\global\long\def\realn{\real^{n}}

\global\long\def\natset{\integerset}

\global\long\def\interger{\integerset}

\global\long\def\integerset{\mathbb{Z}}

\global\long\def\natn{\natset^{n}}

\global\long\def\rational{\mathbb{Q}}

\global\long\def\realPlusn{\mathbb{R_{+}^{n}}}

\global\long\def\comp{\complexset}
 \global\long\def\complexset{\mathbb{C}}

\global\long\def\and{\cap}

\global\long\def\compn{\comp^{n}}

\global\long\def\comb#1#2{\left({#1\atop #2}\right) }

\section{Introduction\label{sec:intro}}

Design of a product/process traditionally involves extensive number
of parameters that highly influence the characteristics of the final
outcome and can only be measured through random experiments. The product
could be a physical product such as a new super alloy with room temperature
superconductivity or a highly tuned machine learning model with human-level
object classification accuracy. Often the experiments, i.e. casting
an alloy with a new composition or training a machine learning model
with new hyperparameters can be costly both in terms of money and
time. Hence it is necessary to reach the target goal with as few experiments
as possible. The relationship between these tunable parameters and
the output characteristics can be represented by a mathematical function.
However, the underlying function behind many complex problems (e.g.
casting an alloy) are not explicitly known (black-box) and expensive
to evaluate. Therefore na\"ive approaches like random or grid search
become inefficient with increasing experimental complexity.

Bayesian optimisation has recently become an established and sample-efficient
approach for the optimisation of such black-box functions and found
several interesting applications in miscellaneous domains: such as
reducing the extensive number of experiments that are usually required
for a good material design \citep{Frazier_Wang_16Bayesian}, designing
high-strength alloys \citep{Vahid_etal_18New}, designing graphene
thermoelectrics \citep{Yamawaki_etal_18Multifunctional}, optimisation
of short polymer fiber synthesis \citep{Li_etal_17Rapid}, designing
renewable energy systems and real-time control \citep{Irani_18Real},
optimisation of robot gait parameters \citep{Lizotte_etal_07Automatic,Martinez_etal_09ABayesian},
environmental monitoring and sensor set selection \citep{Garnett_etal_10Bayesian}
and hyperparameter tuning of machine learning models \citep{Snoek_etal_12Practical}.

Bayesian optimisation is a sequential model based approach where it
initially builds a prior over the unknown objective function and iteratively
chooses the next function evaluation point by optimising a cheap surrogate
function. The surrogate function selects the evaluation points by
keeping a balance between exploiting the region where function value
is expected to be optimal (exploitation) and exploring the region
where the uncertainty about the function value is high (exploration).
With minimum evaluation points, Bayesian optimisation can reach the
global optimum of expensive black-box functions and this makes the
Bayesian optimisation an ideal candidate for the optimisation of such
unknown functions. However, when dealing with big search spaces, the
algorithm visits low function value regions more often before reaching
the optimum. The main reason for this is that Bayesian optimisation
assumes that every point in the search space is equally likely to
be optimum. We call this as cold start phase of Bayesian optimisation.
For complex models with big search spaces and large amount of data,
this cold start phase can be extremely expensive and wasteful of resources
and it is desirable to shorten this cold start phase. 

In order to shorten the cold start phase and thus to improve the optimisation
efficiency, one can make use of any available prior knowledge about
the solution. There are generally two types of prior knowledge which
are used to construct several Bayesian optimisation frameworks in
literature. (1) Prior knowledge of related function evaluation data
from previous optimisations \citep{Bardenet_etal_13Collaborative,Yogatama_Mann_14Efficient,Feurer_etal_15Initializing,Joy_etal_16Flexible,Feurer_etal_18Scalable,Ramachandran_etal_18Selecting,Ramachandran_etal_18Information,Swersky_etal_13Multi},
(2) Prior knowledge about function itself. Since the former methods
require data from similar function optimisations, they become infeasible
for the optimisation of complex models. The latter approach can be
further divided into two types. (i) The function may not be a complete
black-box rather, we may know its behavior particularly e.g. we may
know that the output is monotonically increasing or decreasing with
respect to a particular variable or the output is uni-modal along
certain variable etc, (ii) Experimenter are usually expert of their
field and may have some intuition of ``good'' and ``bad'' regions.
This knowledge can be used as a prior over the optimum location. Although
there has been research on developing methods to incorporate former
type of prior knowledge \citep{Li_etal_18Accelerating,Vellanki_etal_19Bayesian},
there has been no work to incorporate the latter type of prior. Since
the likelihood of searching the optimum is equally likely along the
whole space, obtaining some expert prior information about the optimum
location is, certainly, useful. For this reason, a Bayesian optimisation
framework that can incorporate domain expert prior regarding the optimum
location is \emph{still an open problem}.

In this paper, we propose a new Bayesian optimisation framework that
can incorporate prior knowledge about optima location and thus can
accelerate the optimisation efficiency. To incorporate such knowledge,
we represent the prior knowledge about the optimum location through
a prior distribution. We then use this prior distribution to warp
the search space. We use probability integrity transform to achieve
the warping. The transformation of search space based on the prior
assumption enables the Bayesian optimisation to encourage function
evaluations to come from high likelihood regions as indicated by the
prior. This is achieved through incorporating the warped space information
directly into the Gaussian process by redefining the kernel matrix,
which makes the method agnostic to any acquisition function. We prove
that the new kernel (we call it as \emph{warped kernel}) is a valid
Mercer kernel. In addition, we also show that the convergence of new
Bayesian optimisation framework is always guaranteed with any prior
distribution even if it is misspecified as long as it has a non-zero
likelihood for the true optimum. We validate our new Bayesian optimisation
method through application to optimisation of several benchmark functions
and hyperparameter tuning of two algorithms: Support Vector Machine
(SVM) and Random Forest.

\section{Background}

\sloppy

Let $f:\mathcal{X}\rightarrow\mathbb{R}$ be an expensive black-box
(unknown) function defined over a compact set $\mathcal{X}\subseteq\mathbb{R}^{d}$
and we need to find the maximum of this function. That is, we need
to find a point $\mathbf{x}_{*}$ in $\mathcal{X}$ such that $\mathbf{x}_{*}=\underset{{\bf x\epsilon}\mathcal{X}\subset\mathbb{R}^{d}}{\text{argmax}}f(\mathbf{x})$.
Bayesian optimisation is a sample efficient solution for the optimisation
of these type of expensive black-box functions. Next, we provide an
overview of various components of Bayesian optimisation.

\subsection{Gaussian Process}

As the function $f(\mathbf{x})$ is an unknown, black-box function,
Bayesian optimisation builds a surrogate model of the function, generally
using Gaussian process (GP) \citep{Rasmussen_Williams_06Gaussian},
random forests \citep{Hutter_etal_11Sequential} or Bayesian neural
networks \citep{Springenberg_etal_16Bayesian}. This paper focuses
on GP modeling because the model is robust and at the same time it
is easily analyzable. Let us say, we have some perturbed function
evaluation data comprising of $n$ observations ($\mathcal{D}_{n}=\left(\mathbf{x}_{i},y_{i}\right)\mid i=1,2,...,n$)
and need to make prediction at a new input $\tilde{\mathbf{x}}$ which
is not part of the data, $\mathcal{D}_{n}$. In such a case, GP prior
gives a probabilistic view for every possible functions using a mean
function $\left(\mu:\mathcal{X}\rightarrow\mathbb{R}\right)$ and
a kernel function $\left(k:\mathcal{X}\times\mathcal{X}\rightarrow\mathbb{R}\right)$
\citep{Rasmussen_Williams_06Gaussian}. It is assumed that, the function
$f(\mathbf{x})$ can be sampled as $f({\bf x})\sim\text{GP}(\mu({\bf x}),k({\bf x},{\bf x}^{'}))$
where $\mu({\bf x})$ is a mean function, and $k({\bf x},{\bf x}^{'})$
is a kernel function. Since the mean function can be set to zero without
loss in generality, we assume it as a zero function for the simplicity
of the modeling. Kernel function specifies the correlation between
the function values at any two points ${\bf x}$ and ${\bf x}^{'}$.
There are many kernel functions including squared exponential kernel,
Mat\'ern kernel and linear kernel. As an example, squared exponential
kernel computes the covariance between function values at any two
points $\mathbf{x}$ and $\mathbf{x}'$ as 
\begin{equation}
k({\bf x},{\bf x}^{'})=\sigma_{0}^{2}\exp(-\frac{1}{2l^{2}}\parallel{\bf x}-{\bf x}^{'}\parallel^{2})\label{eq:SE kernel}
\end{equation}
where $l$ is a length scale parameter related to the smoothness of
the function and $\sigma_{0}^{2}$ is the function variance. The function
values $f({\bf x}_{1}),\ldots,f({\bf x}_{n})$ follow a multivariate
Gaussian distribution $\mathcal{N}(0,{\bf K}),$ where ${\bf K}(i,i^{'})=k({\bf x}_{i},{\bf x}_{i^{'}})$
and the noisy function measurements (outputs), $y_{i}=f({\bf x}_{i})+\epsilon_{i}$
where $\epsilon_{i}\sim\mathcal{N}(0,\sigma^{2})$ being the measurement
noise. Given a new query point $\tilde{\mathbf{x}}$ , the function
values $f_{1:n}$ and $f(\tilde{\mathbf{x}})$ are jointly Gaussian
and can be written as 
\begin{equation}
\left[\begin{array}{c}
f_{1:n}\\
f(\tilde{\mathbf{x}})
\end{array}\right]\sim\mathcal{N}\left(0,\left[\begin{array}{cc}
\mathbf{K} & \mathbf{k}\\
\mathbf{k}^{T} & k(\tilde{\mathbf{x}},\tilde{\mathbf{x}})
\end{array}\right]\right)\label{eq:joint Gaussian}
\end{equation}
where $\mathbf{k}=\left[\begin{array}{cccc}
k(\tilde{\mathbf{x}},{\bf x}_{1}) & k(\tilde{\mathbf{x}},{\bf x}_{2}) & \ldots & k(\tilde{\mathbf{x}},{\bf x}_{n})\end{array}\right]$. This leads to the predictive distribution as $p(f(\tilde{\mathbf{x}})\mid\mathcal{D}_{1:n},\tilde{\mathbf{x}})=\mathcal{\mathcal{N}}(\mu_{n}(\tilde{\mathbf{x}}),\sigma_{n}^{2}(\tilde{\mathbf{x}}))$,
where the predictive mean and variance are respectively given by $\mu_{n}(\tilde{\mathbf{x}})=\mathbf{k^{T}\left[K+\sigma^{2}\mathbf{I}\right]^{-1}}f_{1:n}$
and $\sigma_{n}^{2}(\tilde{\mathbf{x}})=k(\tilde{\mathbf{x}},\tilde{\mathbf{x}})-\mathbf{k^{T}\left[K+\sigma^{2}\mathbf{I}\right]^{-1}k}$.

\subsection{Acquisition functions}

So far, we have discussed how to build a probabilistic model for a
smooth function using a Gaussian process and how to update this prior
after gathering new observations. Next, we will give a brief overview
about another important component of Bayesian optimisation called
acquisition function - used to recommend the next promising function
evaluation point. The acquisition function is a cheap surrogate to
the $f({\bf x})$ and is constructed using the posterior information
derived by combining the GP prior with observed data, $\mathcal{D}_{n}$.
While recommending new query points, acquisition function keeps a
trade-off between exploitation and exploration criteria which eventually
guide towards the optimum of the objective function, i.e. for a function
$f(\mathbf{x})$, the next point is sampled as $\mathbf{x}_{n+1}=\underset{{\bf x\epsilon}\mathcal{X}}{\text{argmax }}\alpha\left(\mathbf{x}\mid\mathcal{D}_{n}\right)$.
Popular choices of acquisition functions are Expected Improvement
(EI) \citep{Mockus_etal_78_Theapplication}, GP-UCB \citep{Srinivas_etal_12Information}
and Predictive Entropy Search (PES) \citep{Hernandez-Lobato_etal_14Predictive}.
Since predictive entropy search is computationally expensive, we perform
experiments using EI and GP-UCB.

\textbf{Expected Improvement (EI)} maximizes the expected improvement
over the current best observation. We can write the improvement function
as $\textrm{I}({\bf x})=\max\{0,f_{n+1}({\bf x})-f({\bf x}^{+})\}$
where ${\bf x}^{+}=\underset{{\bf x}_{i}\epsilon{\bf x}_{1:n}}{\text{argmax}}f({\bf x}_{i})$.
If the predicted value is greater than the current best value, $\textrm{I}(\mathbf{x})$
is set to positive, else it is set to zero. Then the new evaluation
point can be written as ${\bf x}=\underset{{\bf x}}{\text{argmax }}\mathbb{E}(\textrm{I}({\bf x})/\mathcal{D}_{1:n})$.
Analytic form of EI can be written as

\begin{equation}
\alpha_{n}(\mathbf{x})=\begin{cases}
(\mu_{n}({\bf x})-f({\bf x}^{+}))\Phi(z)+\sigma_{n}({\bf x})\phi(z) & \textrm{if }\sigma_{n}({\bf x})>0\\
0 & \textrm{if }\sigma_{n}({\bf x})=0
\end{cases}\label{eq:EI}
\end{equation}
where $z=\frac{\mu_{n}({\bf x})-f({\bf x}^{+})}{\sigma_{n}({\bf x})}$.
$\phi(.)$ and $\Phi(.)$ are respectively the probability distribution
function (pdf) and cumulative distribution function (cdf) of a standard
normal distribution.

\textbf{GP-UCB} selects the next observation point based on an upper
confidence bound of predictive distribution. GP-UCB selects the next
function evaluation point by the following maximization 
\begin{equation}
\alpha_{n}(\mathbf{x})=\underset{{\bf x\epsilon}\mathcal{D}}{\text{argmax}}\left(\mu_{n}({\bf x})+\surd\gamma_{n}\sigma_{n}({\bf x})\right)\label{eq:UCB}
\end{equation}
where $\gamma_{n}=2\log\left(2\pi_{n}/\delta\right)+4d\log\left(dnbr\sqrt{\log\left(2da/\delta\right)}\right)$
is $n$-dependent weight derived in \citep{Srinivas_etal_12Information}
to ensure global optimisation. 

\subsection{Related works}

Standard Bayesian optimisation is a sequential approach where it assumes
the likelihood of any point being the function optimum is equal along
the search space. This may cause the algorithm to spend more evaluation
budget near low function values and consequentially make the optimisation
costly especially when dealing with big search spaces. One way to
help the algorithm to jump quickly towards the optima location is
to incorporate any available prior knowledge about the optimum. Such
prior knowledge can be of two types:

1. Prior knowledge from related function evaluations which are already
optimised.

2. Prior knowledge regarding the solution from domain experts.

There are several settings in literature which utilize the knowledge
from previous function optimisations to boost the optimisation performance
of the function (target), $f(\mathbf{x})$. One crucial requirement
while transferring the knowledge from previous function optimisations
is the relatedness between source (previous data) and the target.
Preliminary works in transfer learning for Bayesian optimisation assume
that the source and target are related and transfer the knowledge
without checking the relatedness \citep{Bardenet_etal_13Collaborative,Yogatama_Mann_14Efficient}.
Another similar approach that made assumption regarding the source/target
similarity and transfer the source information in a multi-task setting
is presented in \citep{Swersky_etal_13Multi}. However, performance
of these methods drastically reduce when an unrelated source is used.
To avoid this performance degradation, alternate transfer learning
methods that consider the source/target similarity are proposed in
\citep{Joy_etal_16Flexible,Shilton_etal_17Regret,Ramachandran_etal_18Selecting}.
Later, Feurer et al. \citep{Feurer_etal_15Initializing,Feurer_etal_18Scalable}
proposed meta-learning approaches that use meta-features to measure
the source/target similarity. All the above methods assume function
space similarity rather than a more aligned similarity based on the
optima locations of the function. Recently, Ramachandran et.al. \citep{Ramachandran_etal_18Information}
proposed a transfer learning approach, where similarity across a source
and target is estimated based on their optima locations. This similarity
is then used to bias the distribution of the target function optimum
through a mixture distribution. 

There are many practical situations in which the data from the previous
related optimisations (source functions) may not be available. However,
some useful prior knowledge about the function may still be available.
This type of prior knowledge can be divided into two types. 

(i) Relationship of the function with its variables (e.g. monotonicity).
In a recent method, proposed by Li et al. \citep{Li_etal_18Accelerating}
uses \textquoteleft experimenter hunches\textquoteright{} that defines
the underlying behavior of the experimental system. That is, authors
incorporated monotonic information about the variables using a two-stage
Gaussian process modeling. First stage models the monotonic trend
in the underlying property. Several points are sampled from this GP
and then combined with existing target observations to build a new
GP in second stage. Authors proved the theoretical consistency of
this method by providing a regret bound. Another similar approach
proposed by Vellanki et al. \citep{Vellanki_etal_19Bayesian} incorporates
the shape of the control function as a prior information by modeling
it using a Bernstein polynomial basis. Instead of directly optimising
this control function, authors optimise its Bernstein coefficients. 

(ii) Knowledge about the optima region of the underlying function
from domain knowledge. A domain expert often has a vague understanding
of good and bad regions of the function. Such prior knowledge about
the optima location can be directly used to cut down the search space
away from the optima and thus helps to improve the convergence of
the target function quickly. However, there has been no work to incorporate
this type of prior knowledge and remains\emph{ an open problem}.

\section{Proposed Method}

Our goal is to propose a computationally feasible Bayesian optimisation
method that can incorporate any available prior knowledge about the
function optimum to improve the optimisation efficiency. We aim to
develop a method that is \emph{agnostic} to any acquisition function.
Since Bayesian optimisation generally assumes that any point in the
search space is equally likely to be the optimum, any informative
prior knowledge about the optimum of $f(\mathbf{x})$ is most likely
to give a desirable boost for finding the optima of the function.
A domain expert's vague understanding of good and bad regions can
be approximately represented by a prior distribution, $p(\mathbf{x}_{*})$
that provides a likelihood of any point $\mathbf{x}$ being the optimum.
This prior distribution can then be used to differentially warp the
search space in a manner such that the regions more likely to contain
the optimum are expanded while the regions less likely to contain
the optimum are shrunk. We use a well known technique called probability
integral transform for warping the search space. Next we provide a
discussion about probability integral transform and how we can effectively
incorporate the prior knowledge to improve the optimisation efficiency
by warping the search space.

\subsection{Probability Integral Transform}

Probability integral transform, as the name suggests, transforms any
continuous random variable to a random variable following uniform
distribution. The following \textbf{Lemma \ref{lem:lemma1}} formally
presents the probability integral transform.
\begin{figure}
\centering{}\subfloat[]{\centering{}\includegraphics[width=0.5\textwidth]{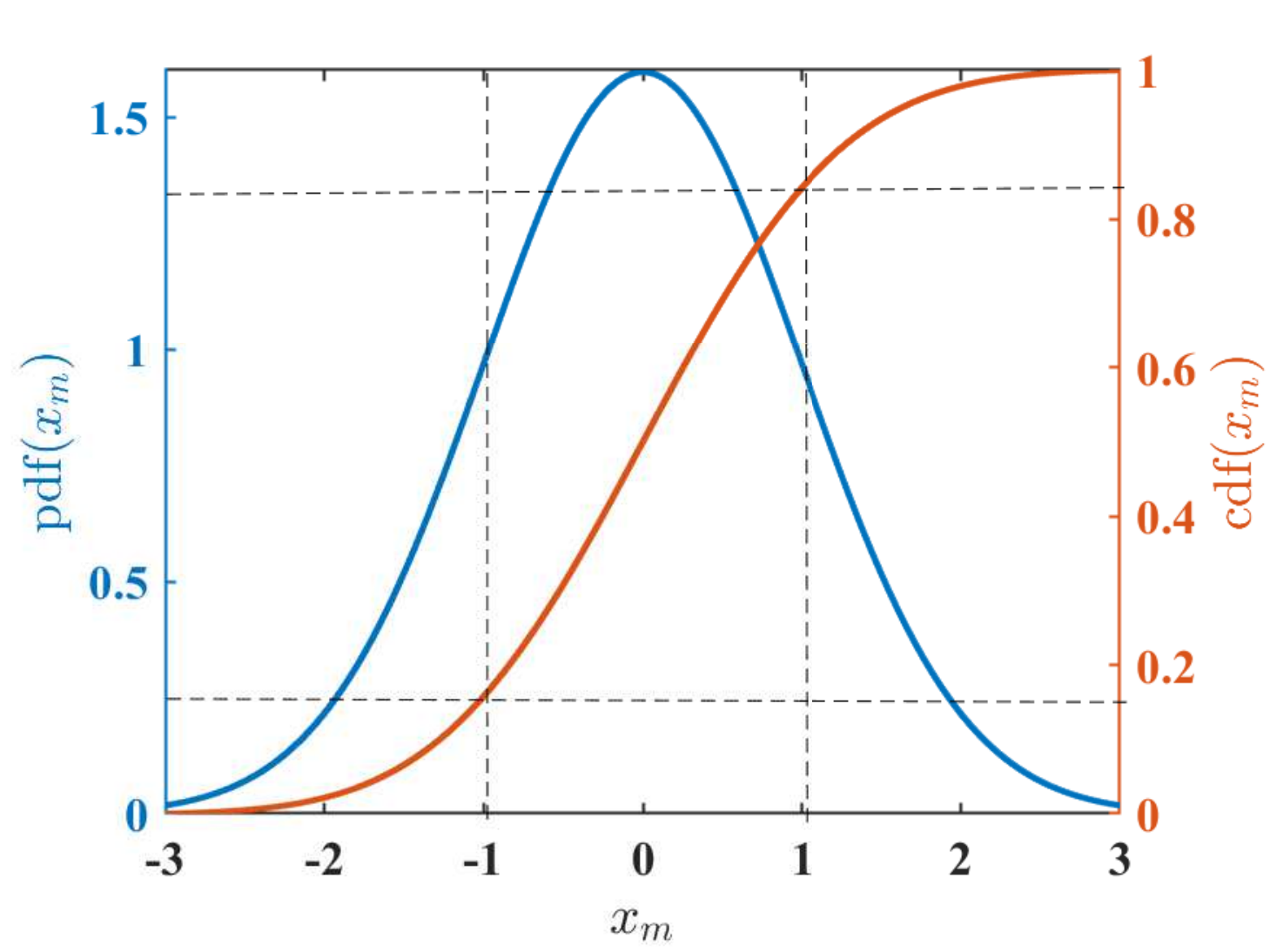}}\caption{Probability distribution function and cumulative distribution function
of an one-dimensional Gaussian function.\label{fig:pdf and cdf - Chapter 5}}
\end{figure}

\begin{lem}
\emph{\label{lem:lemma1}Let $x$ be any real-valued random variable
following a continuous distribution and $\Phi_{\mathbf{x}}$ be its
cumulative distribution function (cdf), then any random variable $\mathbf{u}=\Phi_{\mathbf{x}}(x)$
follows a uniform distribution $\mathcal{U}(0,1)$ defined on} $(0,1)$. 
\end{lem}

\begin{proof}
We have defined, for any random variable, \emph{$\mathbf{u}=\Phi_{\mathbf{x}}(x)$.
}Then 
\begin{align*}
\Phi_{\mathbf{u}}(u) & =\textrm{P}(\mathbf{u}\leq u)=\textrm{P}(\Phi_{\mathbf{x}}(x)\leq u),\textrm{ for any }u\epsilon(0,1)\\
 & =\textrm{P}(x\leq\Phi_{\mathbf{x}}^{-1}(u))=\Phi_{\mathbf{x}}(\Phi_{\mathbf{x}}^{-1}(u))=u
\end{align*}
Taking derivative with respect to $u$, we obtain the probability
distribution function (pdf) of $\mathbf{u}$ as $p(u)=1$, $u\epsilon(0,1)$
which is $\mathcal{U}(0,1)$ with support $(0,1)$ \citep{Angus_94TheProbability}.
\end{proof}
Next we illustrate in Figure \ref{fig:pdf and cdf - Chapter 5}, how
transforming a variable $x_{m}$ through the prior cdf $\Phi_{m}$
leads to expansion and shrinkage of the space along the variable.
We note that based on the $p(x_{*m})$ constructed in Figure \ref{fig:pdf and cdf - Chapter 5},
region $\left[-1,1\right]$ is high likelihood region while $\left[-3,1\right]\cup\left[1,3\right]$
are low likelihood regions. We can see that the region $\left[-1,1\right]$
of $x_{m}$ gets transformed to space $\left[0.18,0.82\right]$ which
is nearly $64\%$ of the space $\left[0,1\right]$ (Expansion). On
the contrary, the region $\left[-3,1\right]\cup\left[1,3\right]$
get transformed to $\left[0,0.18\right]\cup\left[0.82,1\right]$ which
is nearly $36\%$ of the space $\left[0,1\right]$ (Shrinkage). \emph{Essentially,
$\Phi_{m}(x_{m})$ warps the space in such a way so that in the warped
space the distribution of the optimum becomes uniform.}

\subsection{Bayesian Optimisation in the Warped Space}

Next we will discuss how the optimisation performance improves while
transforming the original search space into the warped space. For
that, let us consider a two-dimensional Levy function and plot its
function evaluation points as recommended by Bayesian optimisation
using both search spaces (original and warped) as shown in Figure
\ref{fig:Sampled points Levy- Chapter 5}. In Figure \ref{fig:sampled points both spaces - Chapter 5},
in warped space (blue), Bayesian optimisation is able to quickly identify
the region near the function minima and recommend samples near that.
On the other hand, in original space (red), Bayesian optimisation
initially samples far from the minimum region. However, as iterations
increase, it is able to recommend points near the minimum of the function.
Therefore, Bayesian optimisation in warped space reaches the optimum
faster when compared to working in the original space. 
\begin{figure}
\centering{}\subfloat[\label{fig:2D Levy - Chapter 5}]{\centering{}\includegraphics[width=0.5\textwidth]{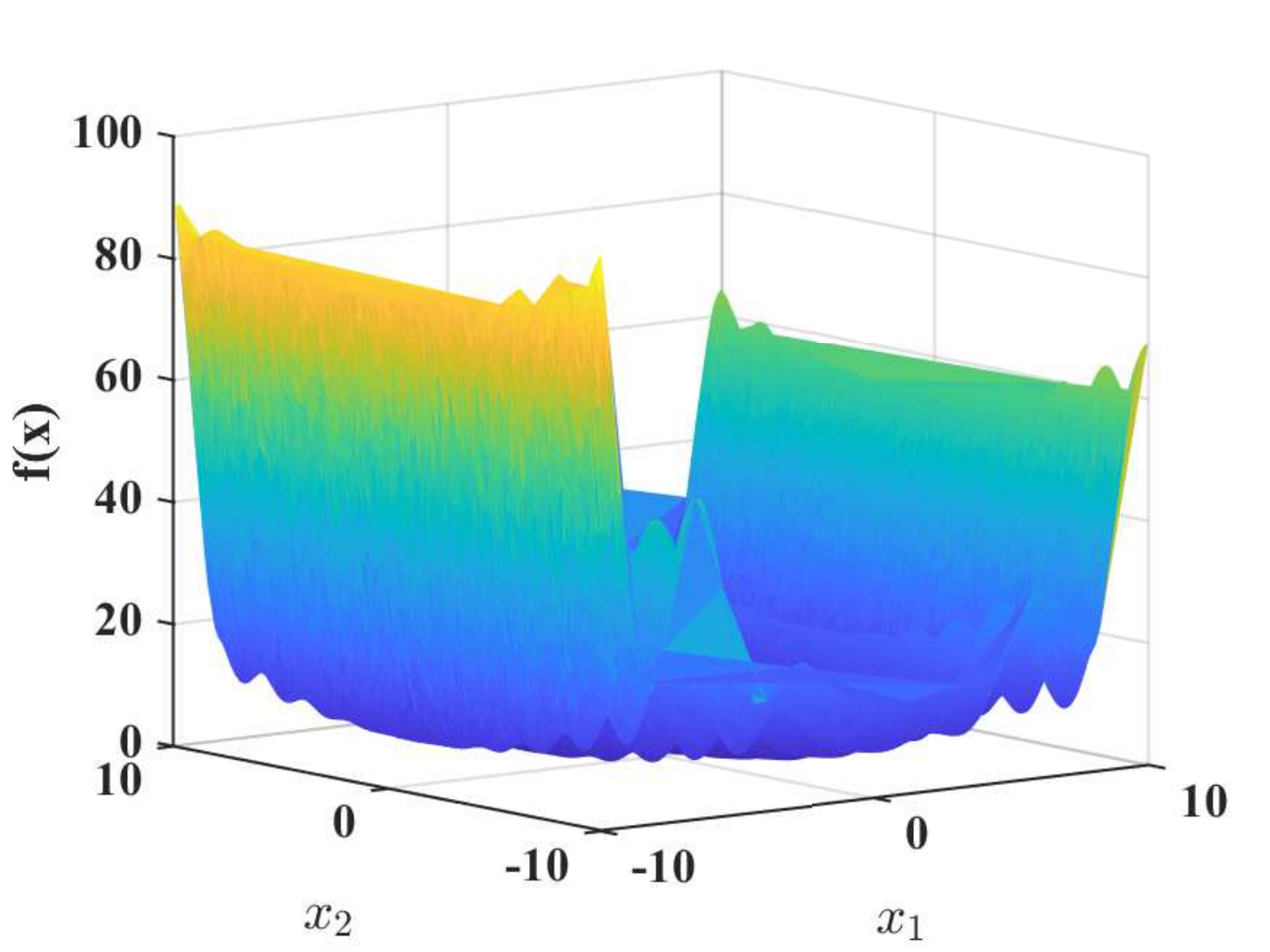}}\subfloat[\label{fig:sampled points both spaces - Chapter 5}]{\centering{}\includegraphics[width=0.5\textwidth]{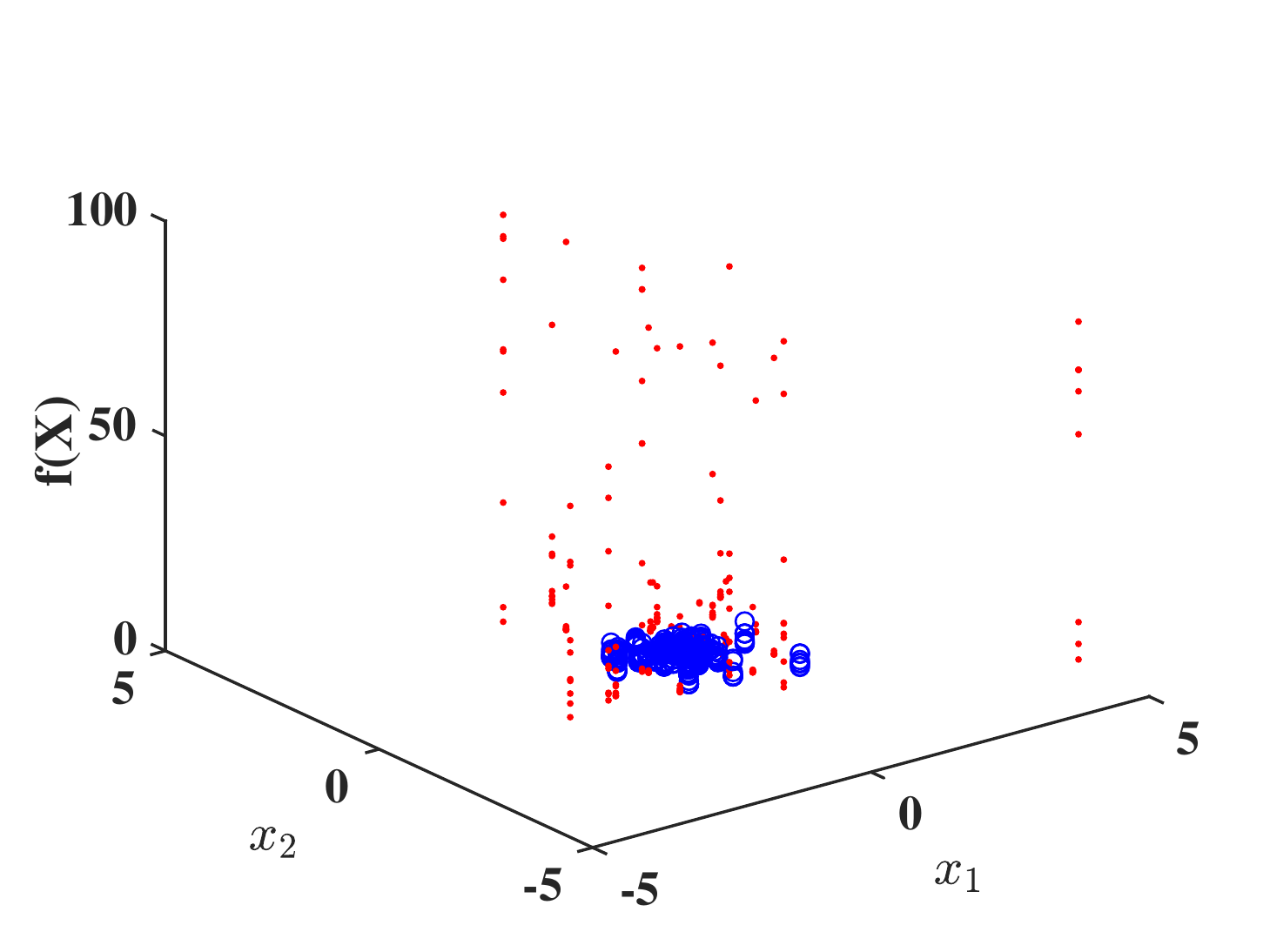}}\caption{$60$ sample points recommended by Bayesian optimisation using both
original and warped space for the two-dimensional Levy function.\label{fig:Sampled points Levy- Chapter 5}}
\end{figure}

Next we describe how to use such warping into Bayesian optimisation.
Particularly, we have two goal here. (i) how to incorporate warping
into Bayesian optimisation efficiently and (ii) how to incorporate
warping into Bayesian optimisation so that it is agnostic to the acquisition
function. Our solution to the above requirement is to intervene into
Gaussian process directly and amend its kernel. For that, let us define
$p(x_{*m})\triangleq p(f(x_{m})\geq f(x_{m}^{'})),\,x_{m}\neq x_{m}^{'}$
as prior distribution over $x_{m}$. The new kernel (\emph{warped
kernel}), $\tilde{k}_{m}(x_{m},x_{m}^{'})$ (can be extended to any
kernel, but here we explain based on the squared exponential kernel)
can be written as
\[
\tilde{k}_{m}(x_{m},x_{m}^{'})=\left(\sigma_{0}^{2}\right)^{d}\exp(-\frac{1}{2l^{2}}\mid\Phi_{m}(x_{m})-\Phi_{m}(x_{m}^{'})\mid^{2})
\]
\begin{equation}
\tilde{k}({\bf x},{\bf x}^{'})=\prod_{m=1}^{d}\tilde{k}_{m}(x_{m},x_{m}^{'})\label{eq: new modified SE kernel}
\end{equation}
where $d$ is the input space dimension and $\Phi_{m}(x_{m})$ and
$\Phi_{m}(x_{m}^{'})$ be the cdf transformations of the points $x_{m}$
and $x_{m}^{'}$ along dimension $m$ respectively. In the following,
we define a\textbf{ Lemma} \textbf{\ref{lem:Lemma2}} that formally
proves the warped kernel, $\tilde{k}({\bf x},{\bf x}^{'})$ is a valid
Mercer kernel, and therefore can be used in Gaussian process to model
the covariance between two points.
\begin{lem}
\emph{\label{lem:Lemma2}The warped kernel $\tilde{k}({\bf x},{\bf x}^{'})$
is a valid Mercer kernel.}
\end{lem}

\begin{proof}
Any kernel, is said to be valid Mercer kernel if it holds two properties.
(i) The kernel function is symmetric and (ii) The kernel is positive
semi-definite. Since warped kernel, $\tilde{k}({\bf x},{\bf x}^{'})$
use the distance $\mid\Phi_{m}(x_{m})-\Phi_{m}(x_{m}^{'})\mid^{2}$
between points $x_{m}$ and $x_{m}^{'}$ along each dimension $m$,
it is invariant to the swapping between $x_{m}$ and $x_{m}^{'}$,
i.e. $\mid\Phi_{m}(x_{m})-\Phi_{m}(x_{m}^{'})\mid^{2}$ is same as
$\mid\Phi_{m}(x_{m}^{'})-\Phi_{m}(x_{m})\mid^{2}$. Thus $\tilde{k}({\bf x},{\bf x}^{'})$
is symmetric ($\tilde{k}({\bf x},{\bf x}^{'})=\tilde{k}({\bf x}^{'},{\bf x})$).
Then the function $\tilde{k}({\bf x},{\bf x}^{'})$ is positive semi-definite
if its associated quadratic form is non-negative, i.e. for all points
${\bf x}_{1},{\bf x}_{2},\ldots,{\bf x}_{n}\epsilon\mathcal{X}$ and
real numbers $a_{1},a_{2},\ldots,a_{n}\epsilon\mathbb{R}$
\[
\sum_{j=1}^{n}\sum_{k=1}^{n}a_{j}a_{k}\tilde{k}({\bf x}_{j},{\bf x}_{k})\geq0
\]
First let us consider $\tilde{k}$ along the dimension $m$. Since
$\tilde{k}_{m}(x_{m},x_{m}^{'})$ can be seen as a function of $\mid\Phi_{m}(x_{m})-\Phi_{m}(x_{m}^{'})\mid$,
it is a stationary function. This allow us to apply Bochner's theorem
to $\tilde{k}_{m}(x_{m},x_{m}^{'})$, i.e. $\tilde{k}_{m}(x_{m},x_{m}^{'})=\mathbb{E}\left[\exp\left(-i\mathbf{w}^{T}(\Phi_{m}(x_{m})-\Phi_{m}(x_{m}^{'}))\right)\right]$
where $\mathbf{w}$ is a random variable of length $n$. Then for
all the points we can write
\begin{align*}
\sum_{j=1}^{n}\sum_{k=1}^{n}a_{j}a_{k}\tilde{k}_{m}(x_{mj},x_{mk}) & =\sum_{j=1}^{n}\sum_{k=1}^{n}a_{j}a_{k}\mathbb{E}\left[\exp\left(-i\mathbf{w}^{T}(\Phi_{m}(x_{mj})-\Phi_{m}(x_{mk}))\right)\right]\\
 & \hspace{-1cm}=\mathbb{E}\left[\sum_{j=1}^{n}\sum_{k=1}^{n}a_{j}\exp\left(-i\mathbf{w}^{T}\Phi_{m}(x_{mj})\right)a_{k}\exp\left(i\mathbf{w}^{T}\Phi_{m}(x_{mk})\right)\right]\\
 & \hspace{-1cm}=\mathbb{E}\left[\mid\sum_{j=1}^{n}a_{j}\exp\left(-i\mathbf{w}^{T}\Phi_{m}(x_{mj})\right)\mid^{2}\right]\geq0
\end{align*}
That means, $\tilde{k}_{m}(x_{m},x_{m}^{'})$ is a positive semi-definite
matrix. Then from Schur product theorem \citep{Schur_11Bemerkungen},
we know that the element-wise product of any positive semi-definite matrices
is also a positive semi-definite matrix, i.e. $\tilde{k}({\bf x},{\bf x}^{'})=\prod_{m=1}^{d}\tilde{k}_{m}(x_{m},x_{m}^{'})$
is also positive semi-definite and thus it is a valid Mercer kernel. 
\end{proof}
The predictive mean and variance according to the new warped kernel,
$\tilde{k}({\bf x},{\bf x}^{'})$ can be respectively written as $\mu_{n}(\tilde{\mathbf{x}})=\mathbf{\tilde{k}^{T}\left[\tilde{K}+\sigma^{2}\mathbf{I}\right]^{-1}}f_{1:n}$
and $\sigma_{n}^{2}(\tilde{\mathbf{x}})=\tilde{k}(\tilde{\mathbf{x}},\tilde{\mathbf{x}})-\mathbf{\tilde{k}^{T}\left[\tilde{K}+\sigma^{2}\mathbf{I}\right]^{-1}\tilde{k}}$
where $\mathbf{\tilde{k}}=\left[\begin{array}{cccc}
\tilde{k}(\tilde{\mathbf{x}},{\bf x}_{1}) & \tilde{k}(\tilde{\mathbf{x}},{\bf x}_{2}) & \ldots & \tilde{k}(\tilde{\mathbf{x}},{\bf x}_{n})\end{array}\right]$ and $\tilde{{\bf K}}(i,i^{'})=\tilde{k}({\bf x}_{i},{\bf x}_{i^{'}})$.
From \textbf{Lemma \ref{lem:Lemma2},} we proved that warped kernel,
$\tilde{k}({\bf x},{\bf x}^{'})$ is a valid Mercer kernel and thus
the function values $f({\bf x}_{1}),\ldots,f({\bf x}_{n})$ follow
a multivariate Gaussian distribution $\mathcal{N}(0,\tilde{{\bf K}})$.
In the following we provide a discussion about the convergence guarantee
of Bayesian optimisation after the incorporation of prior knowledge.

\subsubsection{Discussion about Convergence Guarantee }

\begin{figure}
\centering{}\subfloat[\label{fig:Concentrated prior - Chapter 5}]{\centering{}\includegraphics[width=0.5\textwidth]{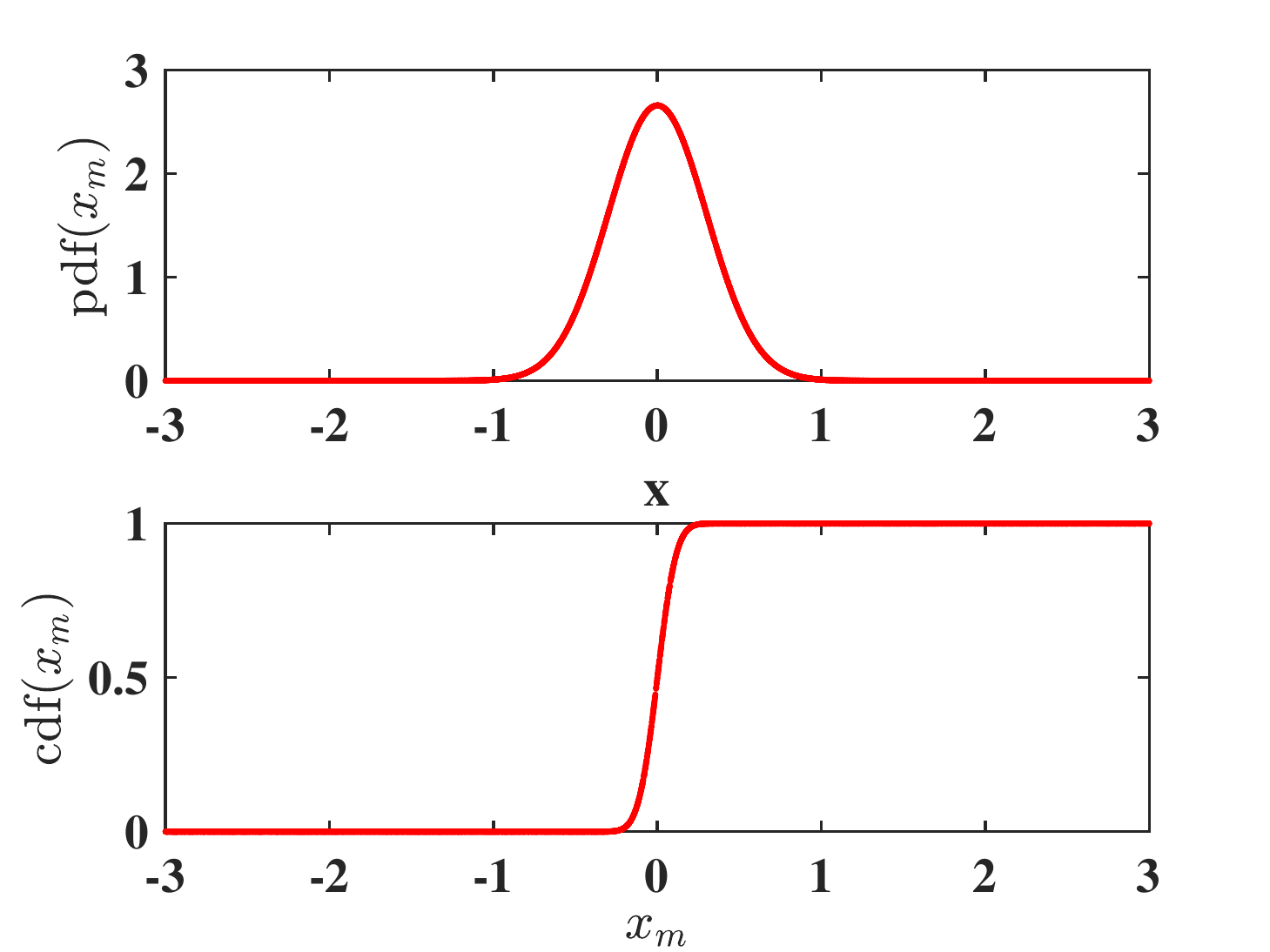}}\subfloat[\label{fig:normal prior - Chapter 5}]{\centering{}\includegraphics[width=0.5\textwidth]{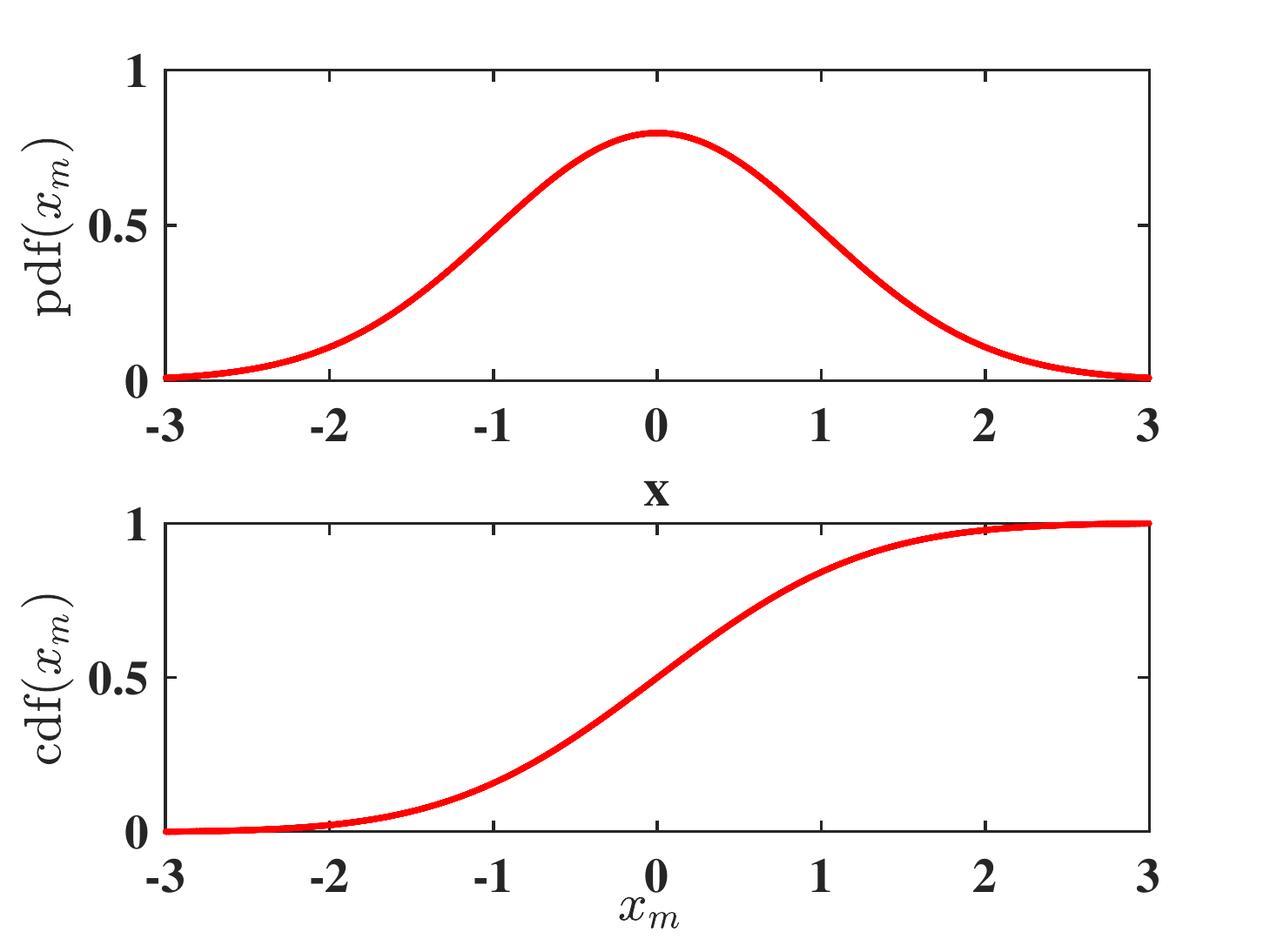}}\caption{(a) When concentrated prior is used for warping the search space.
(b) When widely distributed prior is used for warping the search space.\label{fig:normal and concentrated prior- Chapter 5}}
\end{figure}
To understand the convergence behavior of our proposed method, let
us consider two prior distributions as shown in Figure \ref{fig:normal and concentrated prior- Chapter 5}.
The two priors mainly differ in that the prior in Figure \ref{fig:Concentrated prior - Chapter 5}
has some regions ($\left[-3,1\right]\cup\left[1,3\right]$) with zero
likelihood of optimum whereas the prior in Figure \ref{fig:normal prior - Chapter 5}
has nonzero likelihood of optimum at all points. For the first prior,
any point in the region $[-3,-1]$ gets assigned to point $0$ in
cdf space. Similarly, any point in the region $[1,3]$ gets assigned
to point $1$ in cdf space, i.e. the region $\left[-3,1\right]\cup\left[1,3\right]$
has zero likelihood for being the optima. Thus due to a misspecified
(or misleading) prior, if $x_{*}$ happens to be in the zero likelihood
region, the optimiser will miss it. But, for the second prior in Figure
\ref{fig:normal prior - Chapter 5}, which has nonzero likelihood
of optimum at all points, even if the prior is misspecified, convergence
to $x_{*}$ is always guaranteed as there is always a small probability
of it being the optimum. Since in our algorithm we only use a prior
of second type, our algorithm is guaranteed to converge as it borrows
the convergence behaviour of the traditional Bayesian optimisation.

The proposed Bayesian optimisation method is explained in \textbf{Algorithm
\ref{alg:Proposed Bayesian Optimisation Algorithm}}.
\begin{algorithm}
\begin{enumerate}
\item \textbf{Input: }Initial observations $\mathcal{D}_{n}=\left(\mathbf{x}_{i},\boldsymbol{y}_{i}\right)_{i=1}^{n_{0}}$
\item \textbf{Output: }$\left\{ {\bf x}_{n},y_{n}\right\} _{n=1}^{N}$
\item \textbf{for $n=n_{0},\ldots,N$ do}
\begin{enumerate}
\item Compute the kernel matrix $\tilde{{\bf K}}$ using all $\Phi(\mathbf{x}_{i})'$s.
\item Find ${\bf x}_{n}$ by optimising the acquisition function using
either eq. (\ref{eq:EI}) or eq. (\ref{eq:UCB}).
\item Find the objective function: $y_{n}=f({\bf x}_{n})+\epsilon_{n}$.
\item Augment to the data $\mathcal{D}_{1:n}=\mathcal{D}_{1:n-1}\cup\left({\bf x}_{n},y_{n}\right)$
and update $\tilde{{\bf K}}$.
\end{enumerate}
\item \textbf{end for}
\end{enumerate}
\caption{Bayesian Optimisation Algorithm using warped kernel.\label{alg:Proposed Bayesian Optimisation Algorithm}}
\end{algorithm}

\subsubsection{Some Example Prior Distributions and the Corresponding Warping Functions\label{subsec:prior distributions}}

We assume expert has some knowledge about the optima, which can be
used to construct a prior distribution $g(\mathbf{x}_{*})$. In the
following we provide few transformations based on some standard parametric
distributions that may be relevant to practical applications.

\textbf{1. Normal distribution}: Most common continuous distribution
used for prior assumption is normal distribution as a domain expert
may often be able to guess a good region for function values. Here
we consider truncated normal distribution since the search space in
Bayesian optimisation is defined as a compact space. The cdf of truncated
normal distribution with pdf, $f(\mathbf{x}_{*}|\mu,\sigma^{2})=\frac{1}{\sqrt{2\pi\sigma^{2}}}\exp\left(-\frac{(\mathbf{x}_{*}-\mu)^{2}}{2\sigma^{2}}\right)$
assuming $a<\mathbf{x}_{*}<b$ is given by 
\begin{equation}
\Phi(\mathbf{x}_{*}|\mu,\sigma^{2},a,b)=\frac{\tilde{\Phi}(\mathbf{x}_{*}|\mu,\sigma^{2})-\tilde{\Phi}(a|\mu,\sigma^{2})}{\tilde{\Phi}(b|\mu,\sigma^{2})-\tilde{\Phi}(a|\mu,\sigma^{2})}\label{eq:cdf of truncated normal}
\end{equation}
where $\tilde{\Phi}(\mathbf{x}_{*}|\mu,\sigma^{2})=\frac{1}{2}\left[1+\textrm{erf}\left(\frac{\mathbf{x}_{*}-\mu}{\sigma\sqrt{2}}\right)\right]$.

\textbf{2. Truncated gamma distribution: }In other experiments such
as hyperparameter tuning of machine learning models, the value of
many hyperparameters are constrained to be non-negative. In such cases,
an appropriate continuous distribution for the prior is gamma distribution.
Since, the gamma distribution is defined over $\mathbf{x}_{*}\epsilon\left(0,\infty\right)$,
while we need to supply the prior on a compact space, we consider
a truncated version of gamma distribution defined over $\mathbf{x}_{*}\epsilon\left(a,b\right)$.
The closed form expression for the cdf of the truncated gamma distribution
is as follows. Let the pdf of truncated gamma distribution be defined
by $f(\mathbf{x}_{*};\alpha,\beta)=C\mathbf{x}_{*}^{\alpha-1}e^{-\beta\mathbf{x}_{*}}$
where $\mathbf{x}_{*}\epsilon\left(a,b\right)$. By simple calculation,
$C$ can be estimated to be $C=\frac{\beta^{\alpha}}{\left(\gamma\left(\alpha,\beta b\right)-\gamma\left(\alpha,\beta a\right)\right)}$
where $\gamma\left(\alpha,\beta b\right)$ and $\gamma\left(\alpha,\beta a\right)$
are the lower incomplete gamma functions. Now for the truncated gamma
distribution with pdf, $f(\mathbf{x}_{*};\alpha,\beta)=C\mathbf{x}_{*}^{\alpha-1}e^{-\beta{\bf \mathbf{x}_{*}}}$,
$\mathbf{x}_{*}\epsilon\left(a,\mathbf{x}_{*}'\right)$, its cdf can
be written as 

\begin{equation}
\Phi(\mathbf{x}_{*}';\alpha,\beta)=\frac{\gamma\left(\alpha,\beta\mathbf{x}_{*}'\right)-\gamma\left(\alpha,\beta a\right)}{\gamma\left(\alpha,\beta b\right)-\gamma\left(\alpha,\beta a\right)}\label{eq:cdf truncated gamma}
\end{equation}

\section{Experiments}

We experiment the proposed Bayesian optimisation method on some benchmark
functions as well as on various real data experiments. \textbf{Since
in our experiments the aim is to find the global minima of the function,
instead of $f(\mathbf{x})$ we passed $1-f(\mathbf{x})$ to the optimiser.}
The experiments on benchmark functions are designed to illustrate
the accelerated convergence of our method to reach the minimum of
the function due to prior knowledge. In real data experiments, we
tune the hyperparameters of two algorithms: Support Vector Machine
(SVM) and Random Forest. For these algorithms, we use the prior knowledge
gained by previous studies \citep{Van_Hutter_18Hyperparameter}. We
demonstrate the superiority of our method by comparing it with following
two baselines.

1. Standard BO - This baseline is the standard Bayesian optimisation
algorithm.

2. Prior based Search - In this baseline, function evaluation points
are randomly sampled from the prior (see section \ref{subsec:prior distributions}). 

\subsection{Experimental setting}

In all the experiments, square-exponential kernel is used for Gaussian
process modeling and maximum likelihood estimate is used for estimating
Gaussian process kernel length scale. All the results are averaged
over $10$ runs with random initial values. 

\subsection{Benchmark functions}

\subsubsection{Gaussian function:}

We generated a $3$-dimensional Gaussian function with the following
form 
\[
f(\mathbf{x})=1-\text{exp}\left(-\frac{1}{2}(\mathbf{x}-\boldsymbol{\mu})(\mathbf{x}-\boldsymbol{\mu})^{T}\right)
\]
where
\begin{figure*}
\begin{centering}
\subfloat[$5\%$ off]{\centering{}\includegraphics[width=0.5\textwidth]{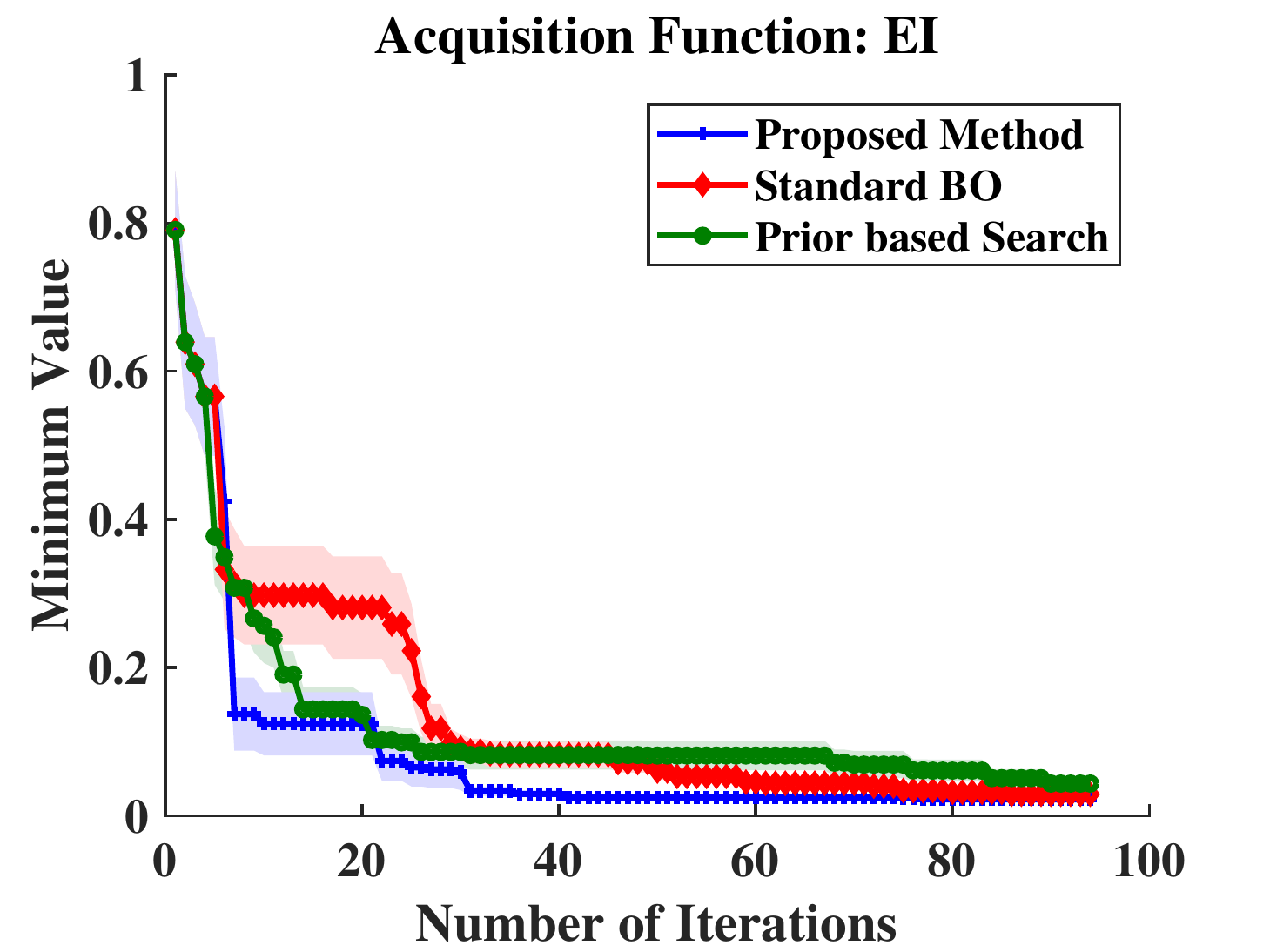}}\subfloat[$5\%$ off]{\centering{}\includegraphics[width=0.5\textwidth]{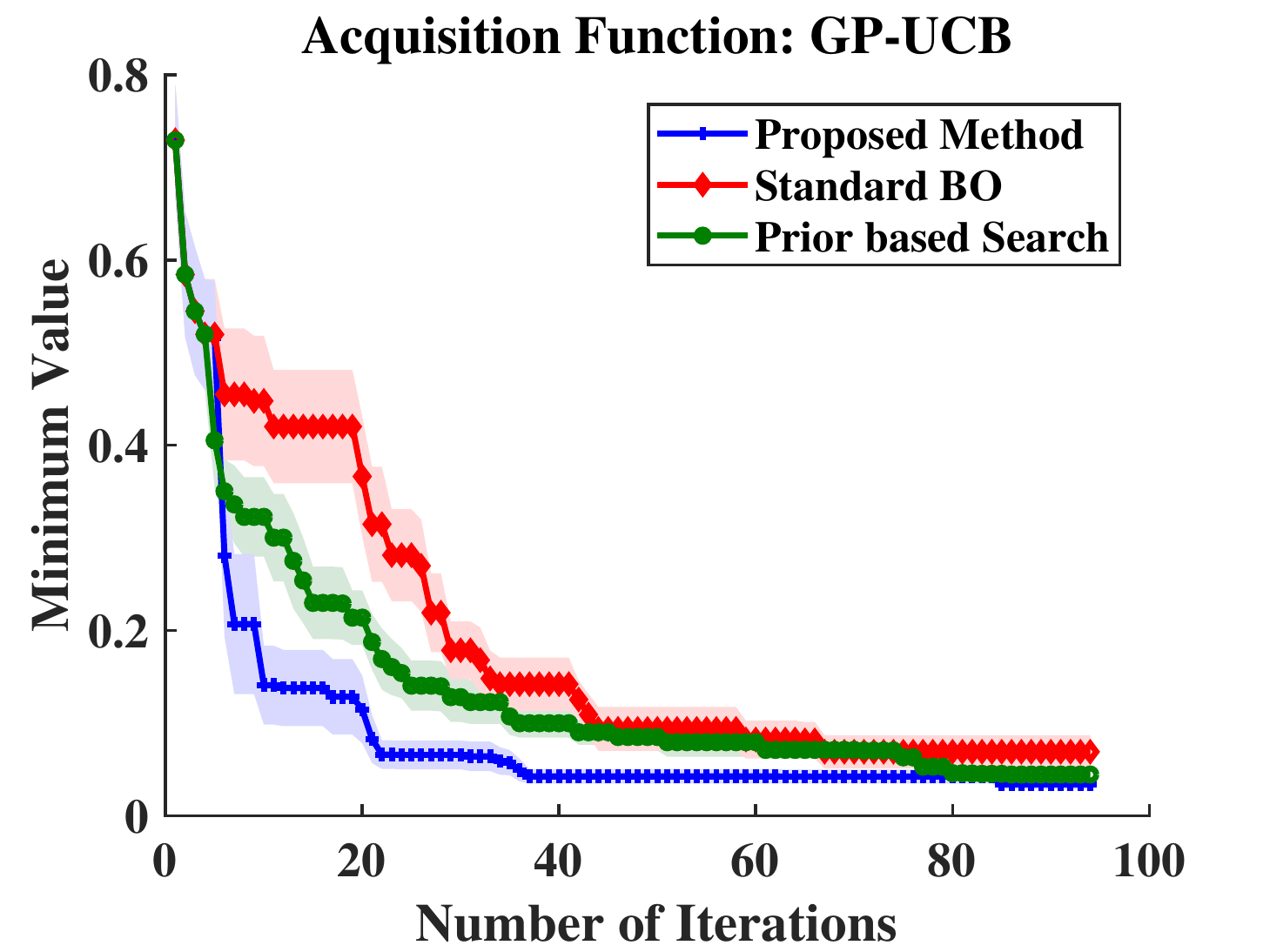}}
\par\end{centering}
\begin{centering}
\subfloat[$10\%$ off]{\centering{}\includegraphics[width=0.5\textwidth]{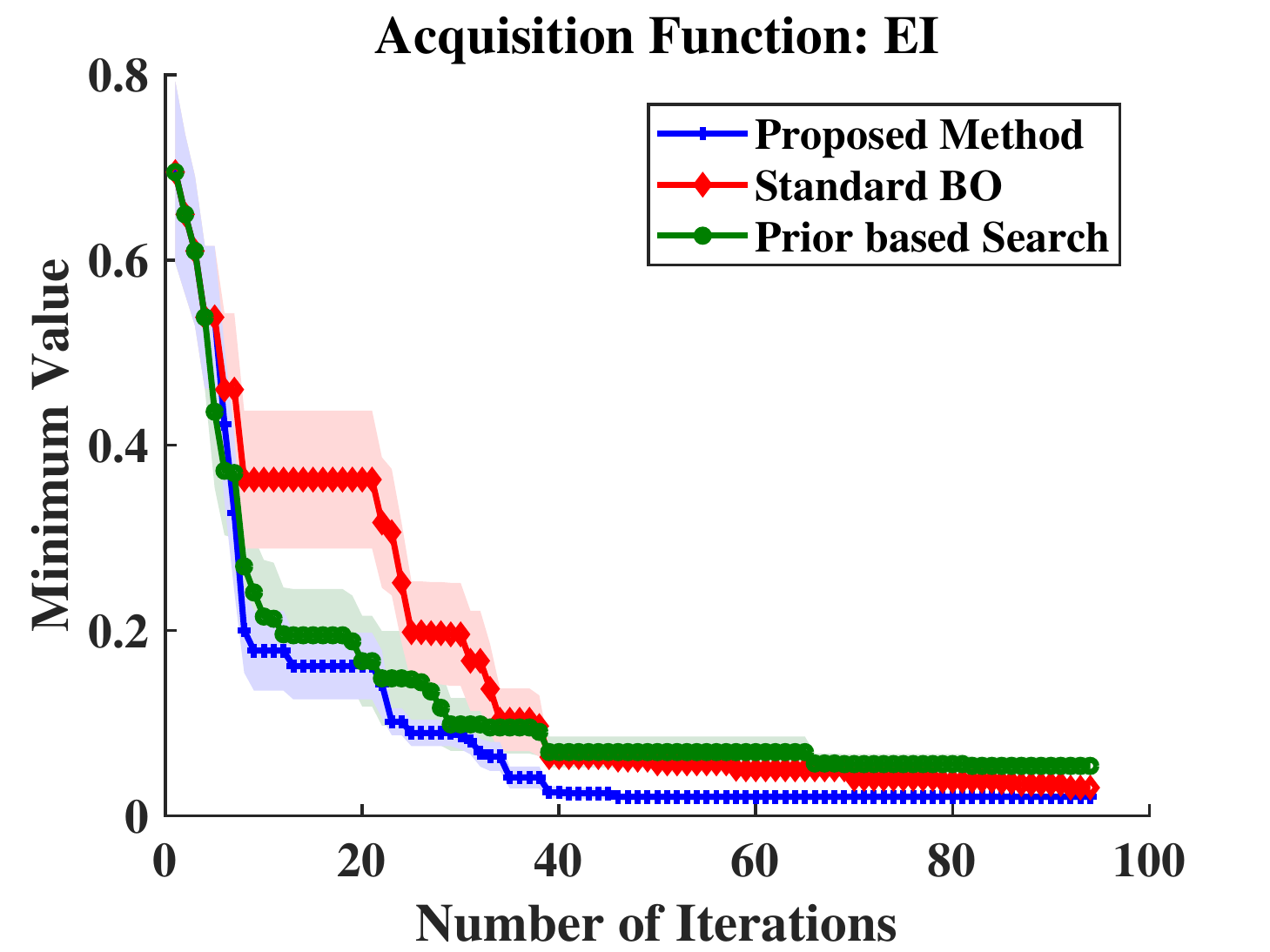}}\subfloat[$10\%$ off]{\centering{}\includegraphics[width=0.5\textwidth]{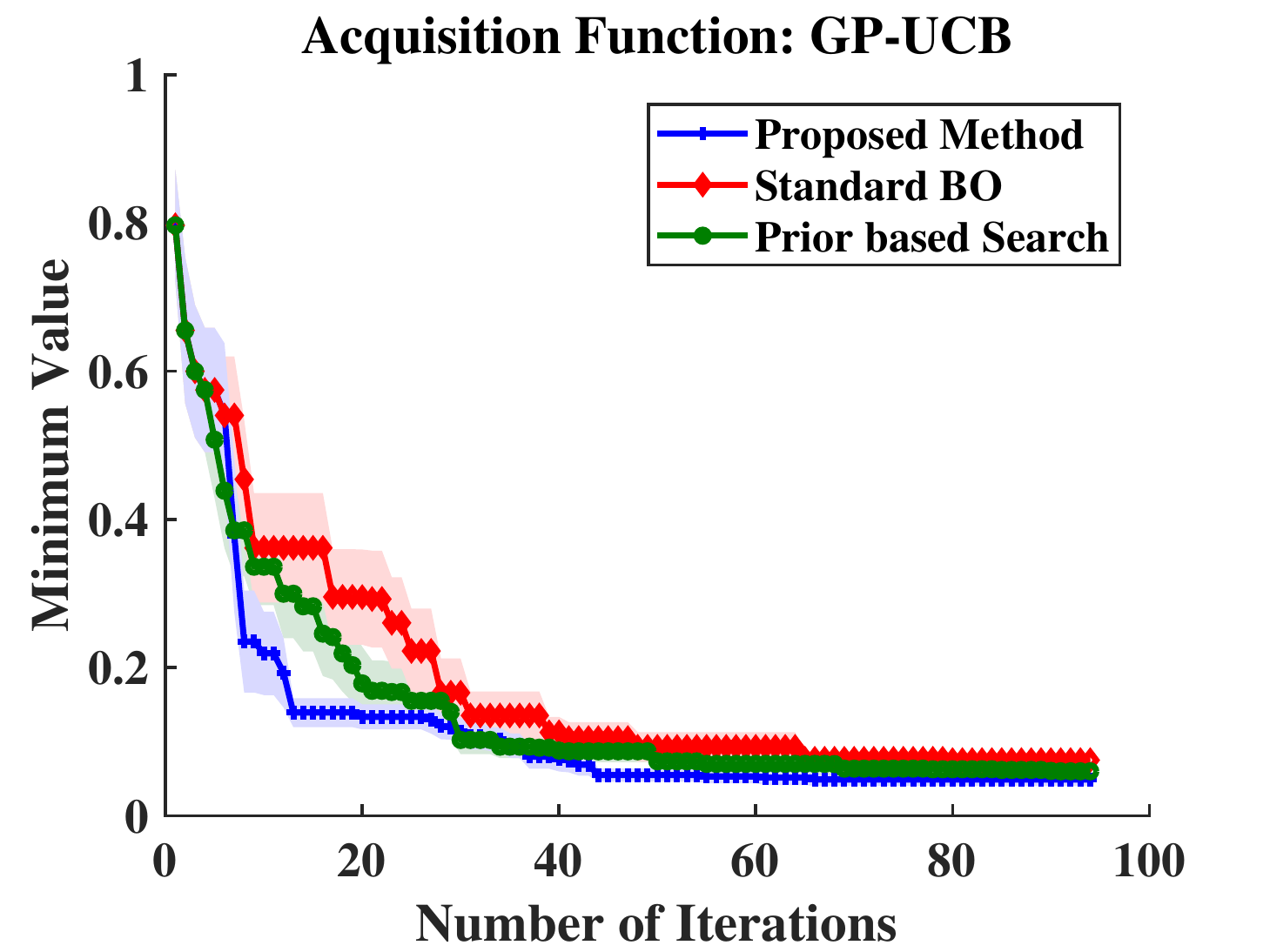}}
\par\end{centering}
\begin{centering}
\subfloat[$20\%$ off]{\centering{}\includegraphics[width=0.5\textwidth]{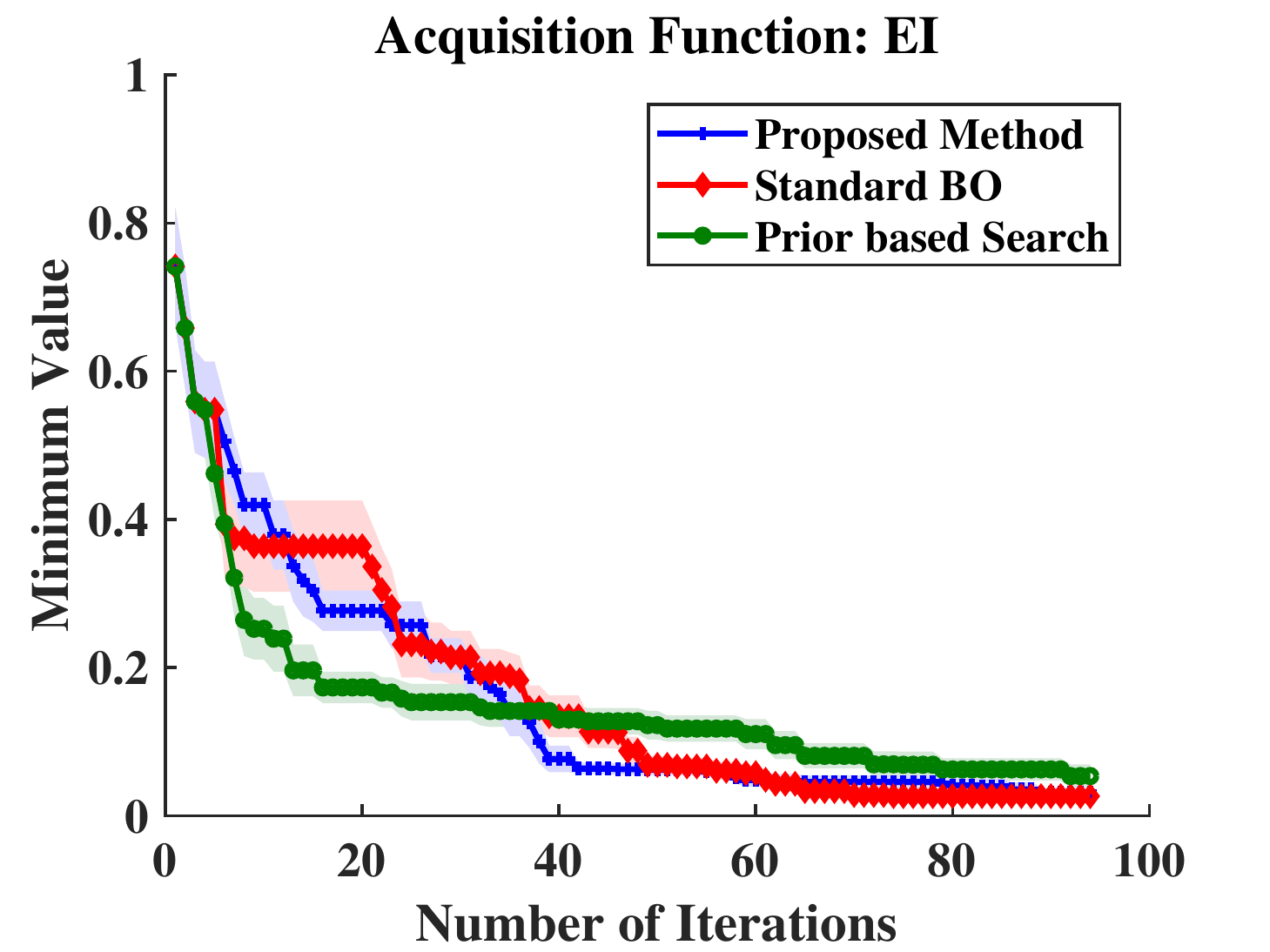}}\subfloat[$20\%$ off]{\centering{}\includegraphics[width=0.5\textwidth]{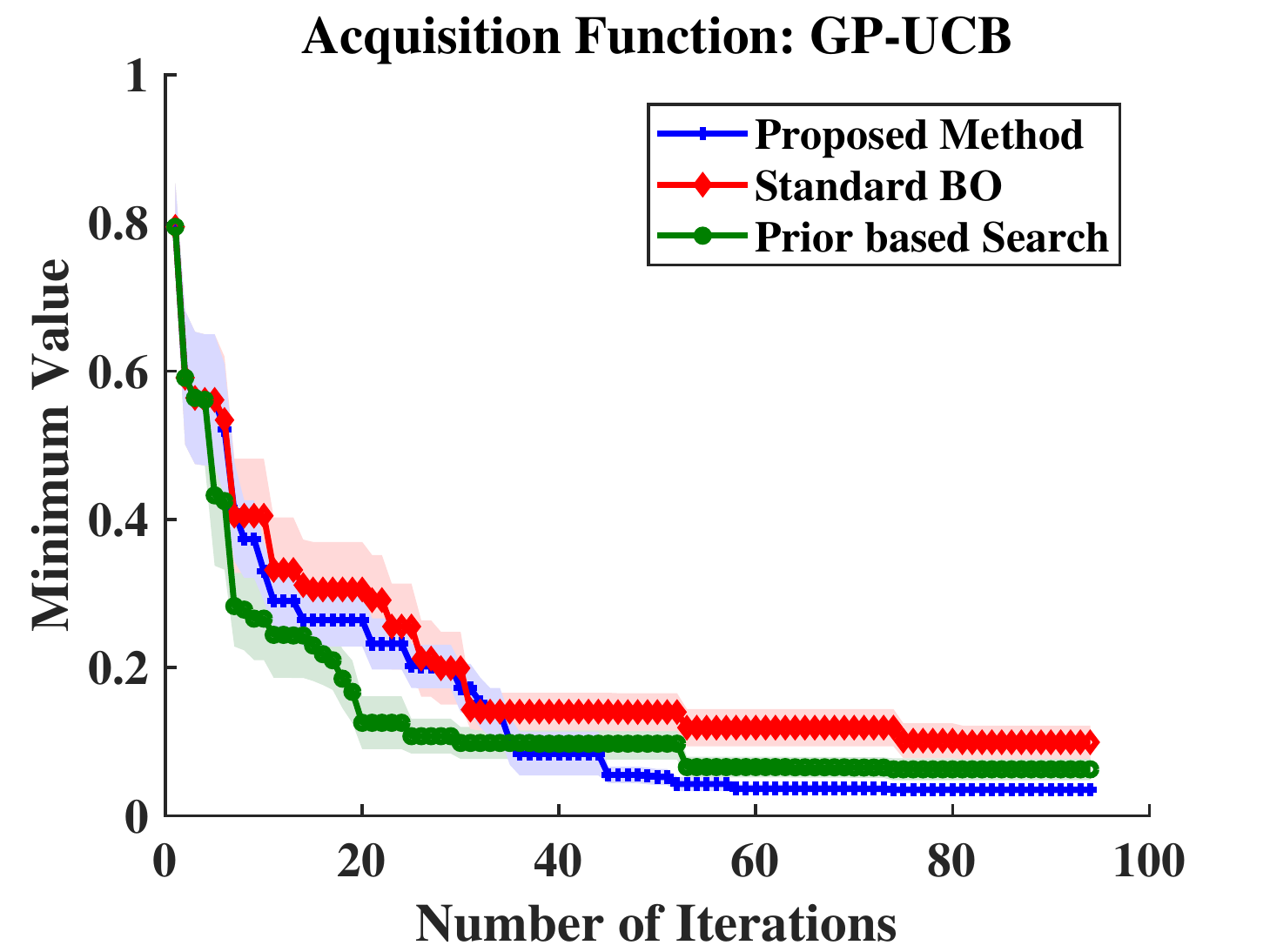}}
\par\end{centering}
\caption{Gaussian function: Minimum value reached vs optimisation iterations
(using both Expected improvement and GP-UCB). (a) \& (b) Prior mean
is $5\%$ off from the true minimum $(0.2)$. (c) \& (d) Prior mean
is $10\%$ off from the true minimum. (e) \& (f) Prior mean is $20\%$
off from the true minimum.\label{fig:Gaussian-function - Chapter 5}}
\end{figure*}
\begin{figure*}
\begin{centering}
\subfloat[$5\%$ off]{\centering{}\includegraphics[width=0.5\textwidth]{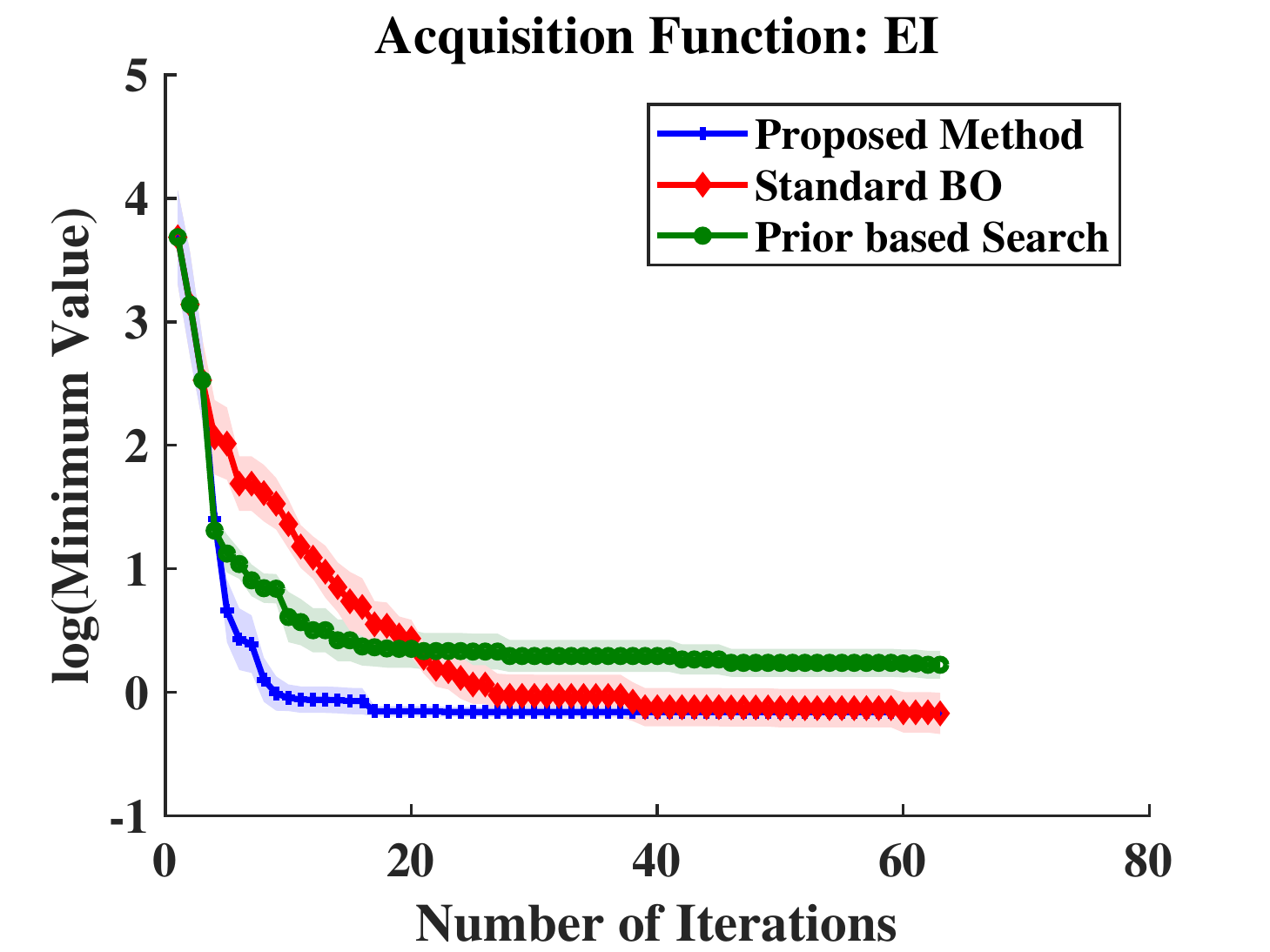}}\subfloat[$5\%$ off]{\centering{}\includegraphics[width=0.5\textwidth]{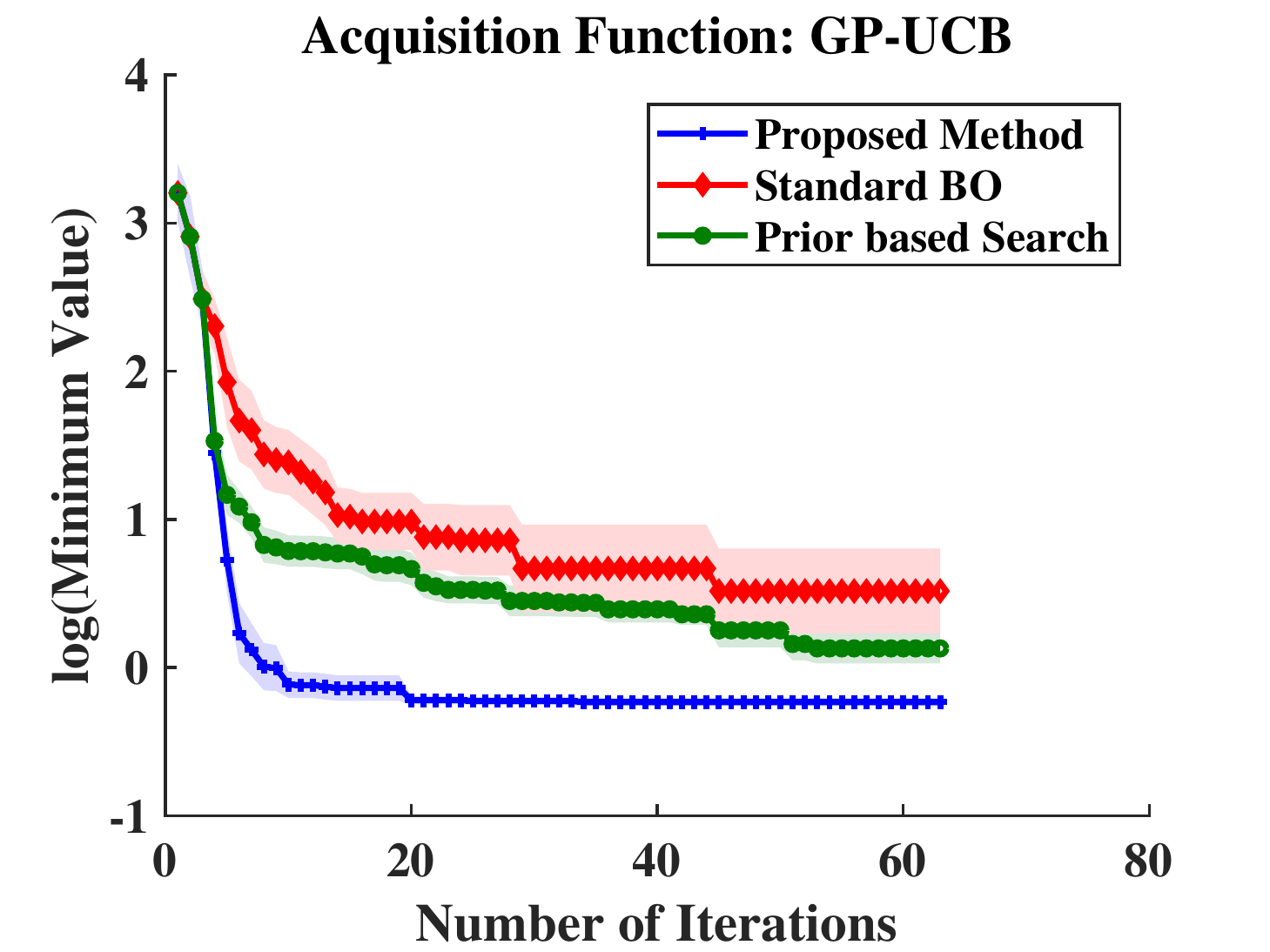}}
\par\end{centering}
\begin{centering}
\subfloat[$10\%$ off]{\centering{}\includegraphics[width=0.5\textwidth]{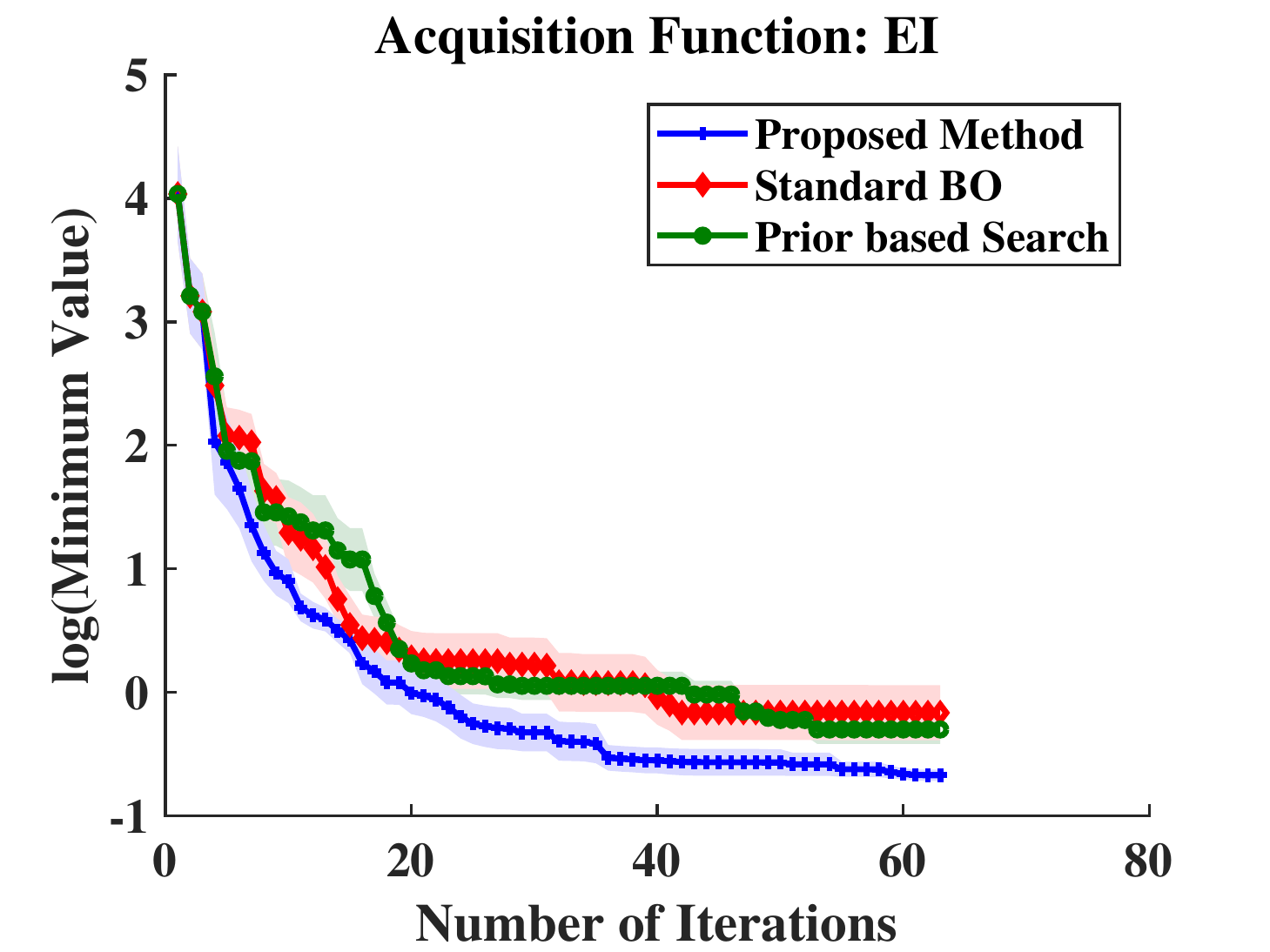}}\subfloat[$10\%$ off]{\centering{}\includegraphics[width=0.5\textwidth]{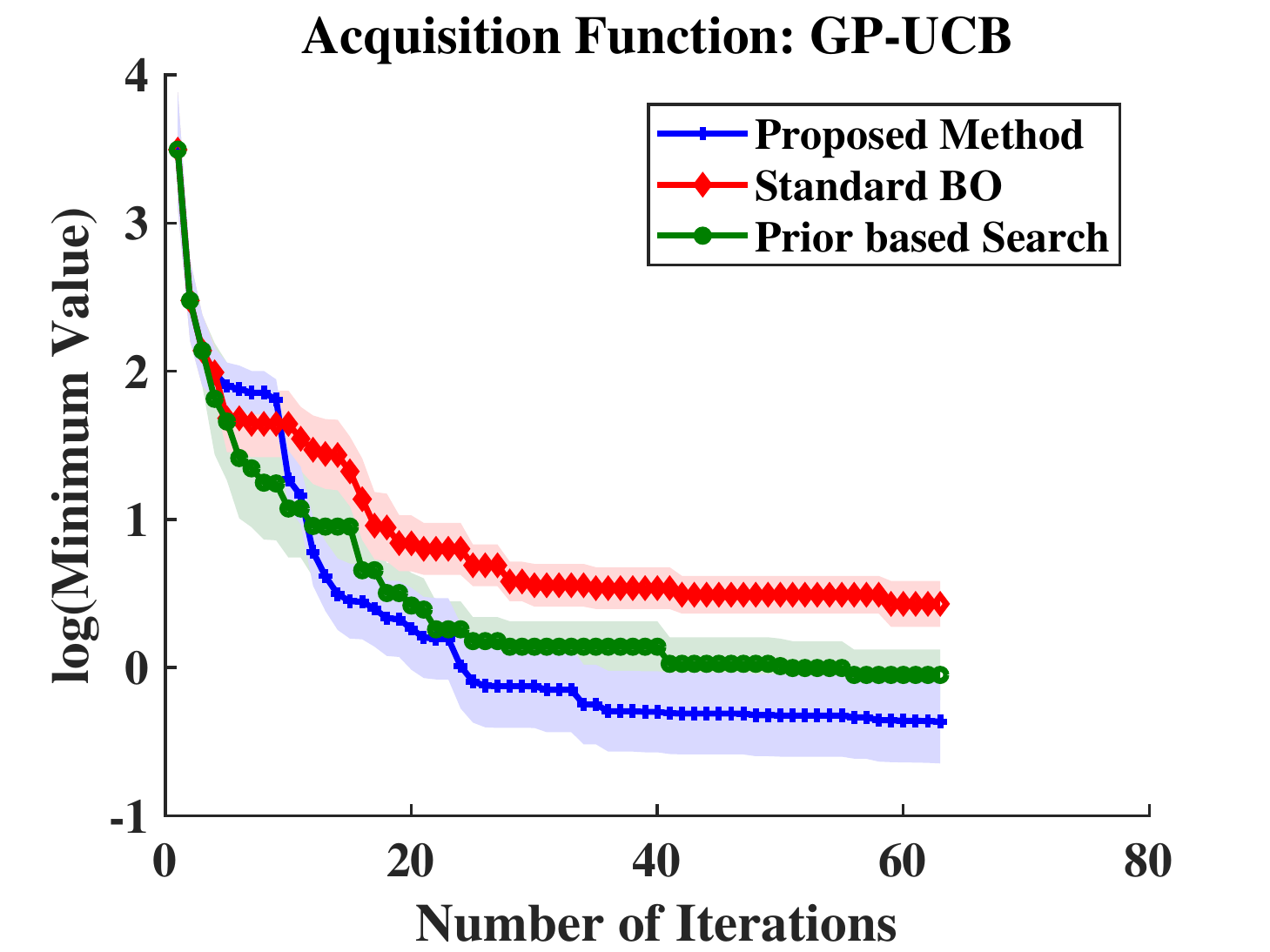}}
\par\end{centering}
\begin{centering}
\subfloat[$20\%$ off]{\centering{}\includegraphics[width=0.5\textwidth]{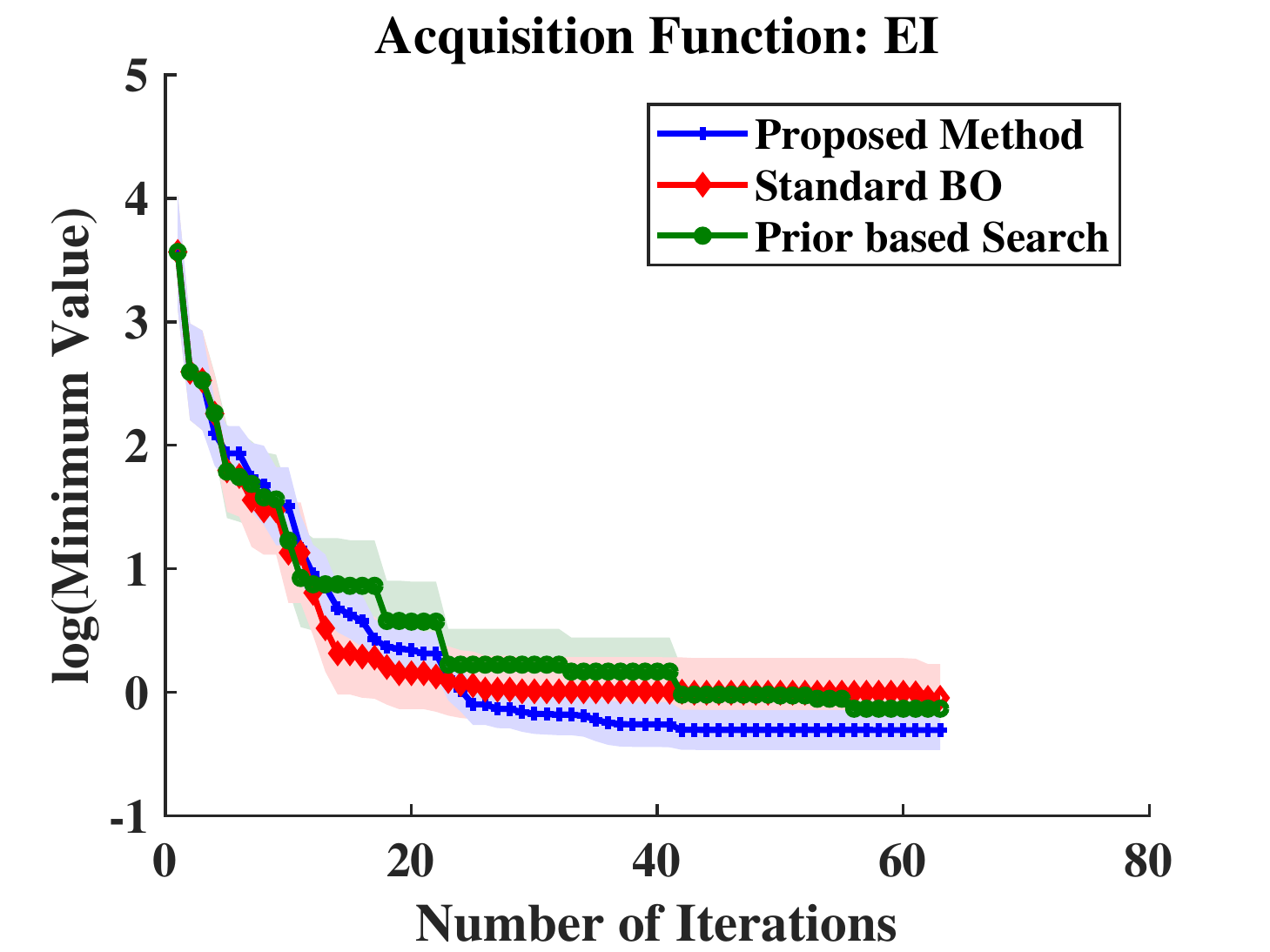}}\subfloat[$20\%$ off]{\centering{}\includegraphics[width=0.5\textwidth]{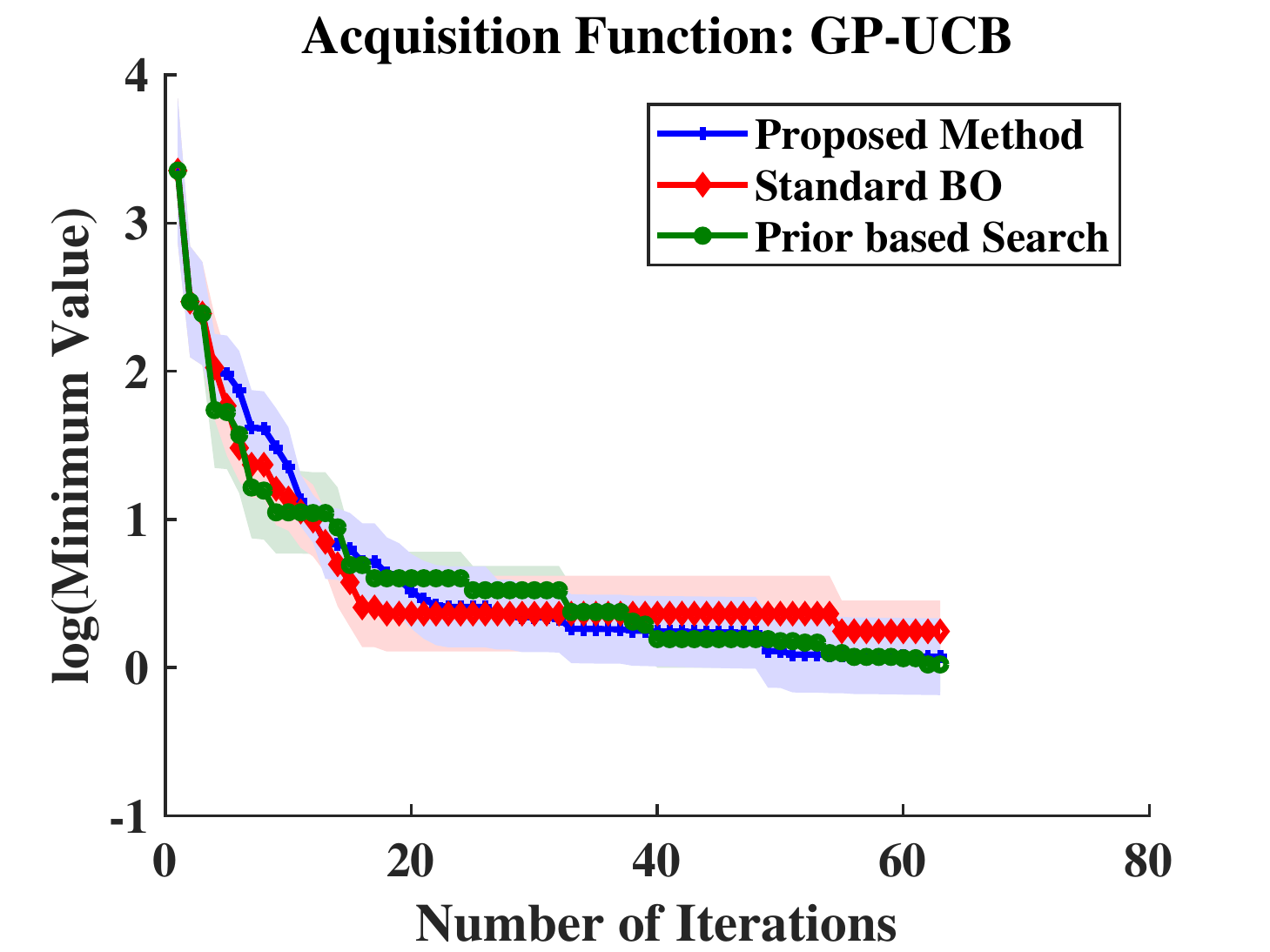}}
\par\end{centering}
\caption{Branin function: Minimum value reached vs optimisation iterations
(using both Expected improvement and GP-UCB). (a) \& (b) Prior mean
is $5\%$ off from the true minimum. (c) \& (d) Prior mean is $10\%$
off from the true minimum. (e) \& (f) Prior mean is $20\%$ off from
the true minimum. \label{fig:Branin-function - Chapter 5 }}
\end{figure*}
 $\boldsymbol{\mu}=\left[0.2,0.2,0.2\right]$. Here truncated normal
distribution is selected as prior distribution. To demonstrate how
optimisation performance changes with prior mean, we shift the mean
of the prior by different amounts from the true minimum $(0.2)$.
In particular, we shift the prior mean from the true function minimum
location by $5\%$, $10\%$ and $20\%$ of the search space. The standard
deviation is set to $1$ for all these cases to cover the full search
space. Figure \ref{fig:Gaussian-function - Chapter 5} shows the minimum
value obtained with respect to iterations for each selected prior
(using both Expected Improvement and GP-UCB acquisition functions).
Each method starts with the same four random observations from $\left[-2,2\right]$
along each dimension. The performance of proposed Bayesian optimisation
method clearly outperforms the standard Bayesian optimisation method
when the prior mean is close $(5\%/10\%\,\textrm{off})$ to the true
minimum. When the prior mean moves away from the true minimum ($20\%\,\textrm{off}$),
the performance of the proposed method starts to degrade and becomes
similar to standard Bayesian optimisation.

\subsubsection{Branin function:}

Next we try to find the minimum of Branin function defined on the
region $\mathbf{x}_{1}\epsilon\left(-5,10\right)$, $\mathbf{x}_{2}\epsilon\left(0,15\right)$.
Branin function can be mathematically written as 
\[
f(\mathbf{x})=a\left(\mathbf{x}_{2}-b\mathbf{x}_{1}^{2}+c\mathbf{x}_{1}-r\right)^{2}+s\left(1-t\right)\cos\left(\mathbf{x}_{1}\right)+s
\]
where $a=1,b=5.1/(4\pi^{2}),c=5/\pi,r=6,s=10,t=1/(8\pi)$. Here also
we selected truncated normal distribution as prior distribution. Once
again we analyze three cases: prior mean off by $5\%$, $10\%$ and
$20\%$ of the search space. For the first case, we simulate a case
with high confidence and thus set the prior standard deviation to
$0.25$. For the other two cases, we simulate medium confidence and
thus set the standard deviation to $4$. Minimum value obtained with
respect to iterations using both EI and GP-UCB acquisition functions
for all three prior means are shown in Figure \ref{fig:Branin-function - Chapter 5 }.
With good prior, our method is able to converge faster than the baselines.
When the prior mean is moving away from the true minimum, the performance
of the proposed method starts to degrade as expected.

\subsection{Hyperparameter Tuning}

The performance of various machine learning algorithms highly depend
on good values of their hyperparameters. In \citep{Van_Hutter_18Hyperparameter},
authors proposed a framework based on functional ANOVA \citep{Hoos_etal_14Anefficient}
to determine the most important hyperparameters of various algorithms
(including support vector machines and random forest) and deduce priors
over which values of these hyperparameters yield good performance.
The most important hyperparameters are inferred by experimenting functional
ANOVA in $100$ different datasets available on OpenML \citep{Vanschoren_etal_14Openml}.
We use these priors over hyperparameters in our proposed Bayesian
optimisation method. In our experiments, we tune the hyperparameters
of two algorithms: (1) Support Vector Machine (SVM) with two different
kernel types: (i) Radial Basis Function (RBF) and (ii) Sigmoid; (2)
Random forest.
\begin{table}
\centering{}%
\begin{tabular}{|c|c|c|}
\hline 
kernel/dataset &  `german numer' & `four class'\tabularnewline
\hline 
\hline 
RBF & $\alpha=2$, $\beta=0.5$ & $\alpha=2$, $\beta=0.5$\tabularnewline
\hline 
sigmoid & $\alpha=2$, $\beta=0.5$ & $\alpha=1$, $\beta=0.5$\tabularnewline
\hline 
\end{tabular}\caption{The values of shape parameter $\alpha$ and inverse scale parameter
$\beta$ of the prior truncated gamma distribution used when tuning
hyperparameters of SVM with RBF and sigmoid kernel on two different
binary classification datasets - `german numer' and `four class'.\label{tab:Parameter values of truncated gamma distribution} }
\end{table}
\begin{figure*}
\begin{centering}
\subfloat[\label{fig:Hyperparameter-tuning:exp1 - SVMRBF EI - Chapter 5}]{\centering{}\includegraphics[width=0.4\textwidth]{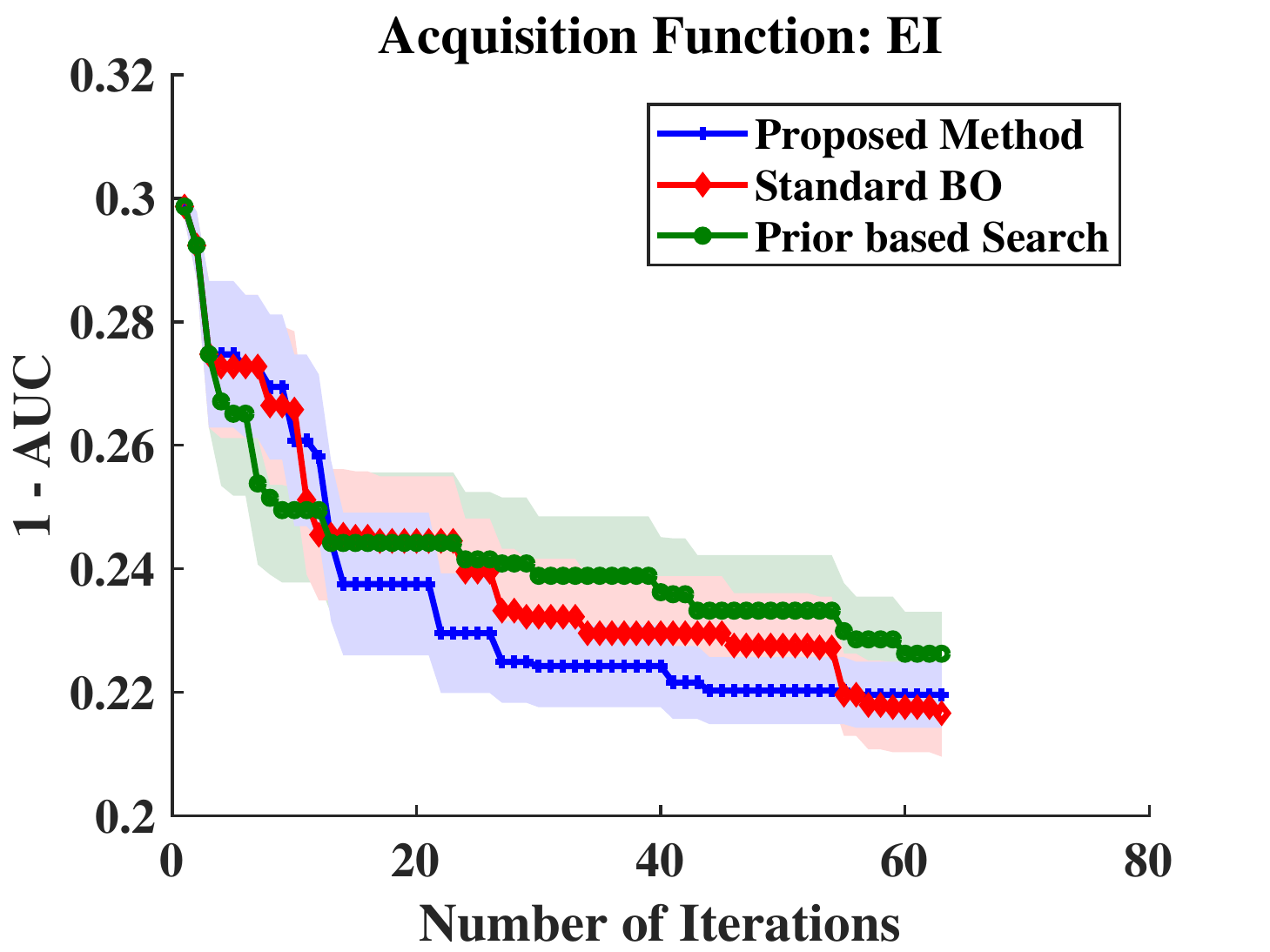}}\subfloat[\label{fig:Hyperparameter-tuning:exp1 - SVMRBF GP-UCB - Chapter 5}]{\centering{}\includegraphics[width=0.4\textwidth]{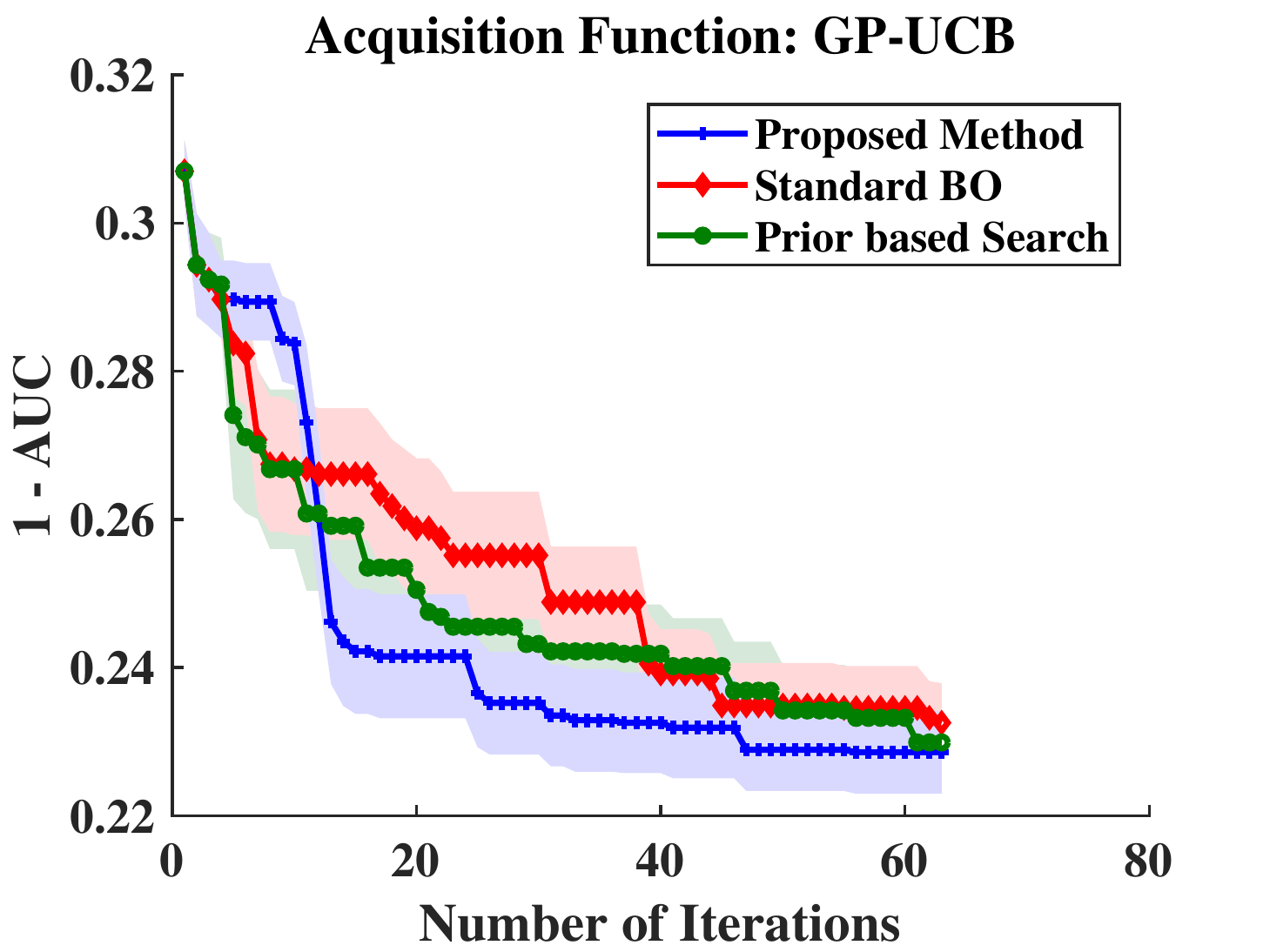}}
\par\end{centering}
\begin{centering}
\subfloat[\label{fig:Hyperparameter-tuning: exp2 SVMRBF EI - Chapter 5}]{\centering{}\includegraphics[width=0.4\textwidth]{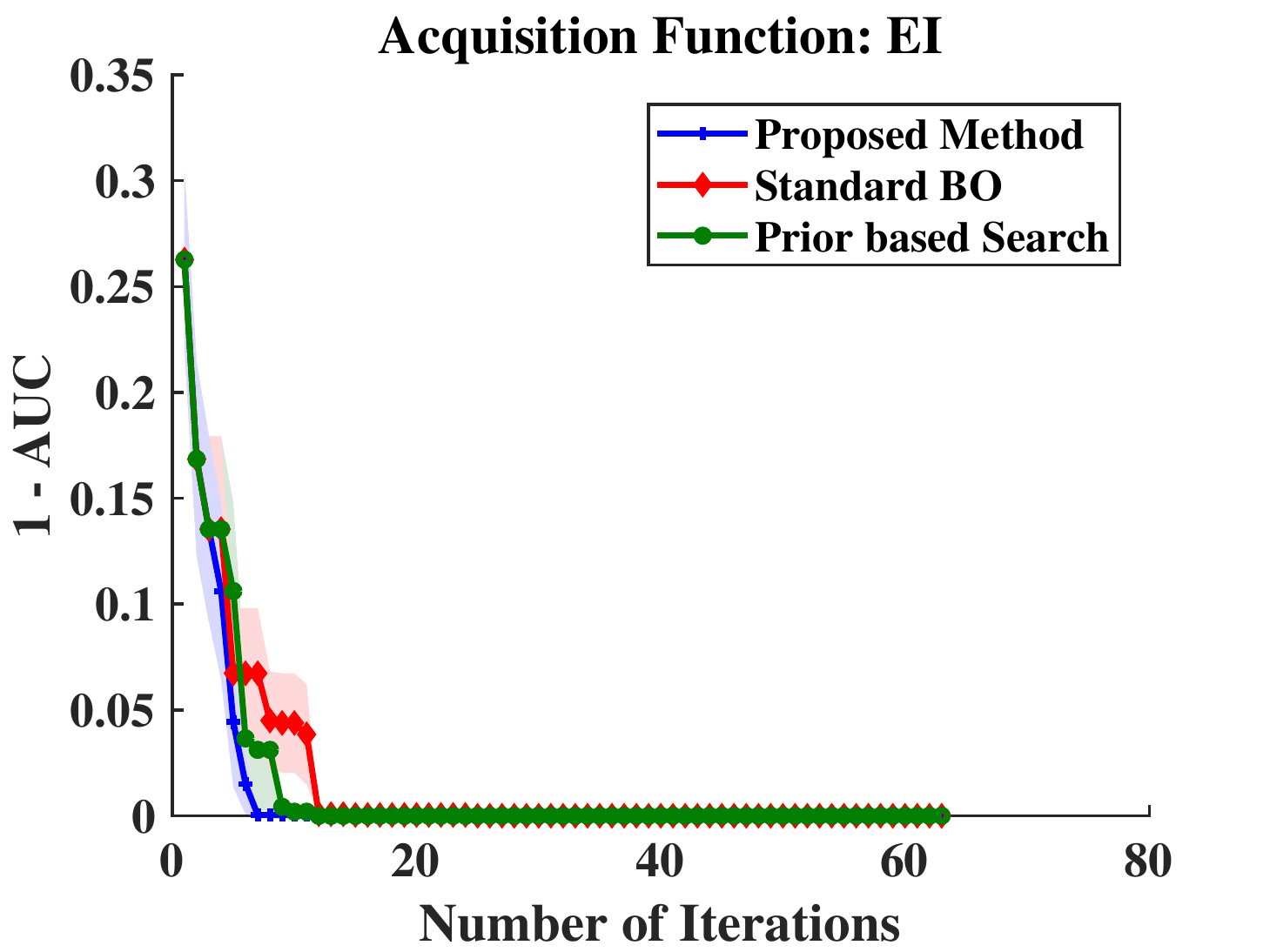}}\subfloat[\label{fig:Hyperparameter-tuning: exp2 SVMRBF GP-UCB - Chapter 5}]{\centering{}\includegraphics[width=0.4\textwidth]{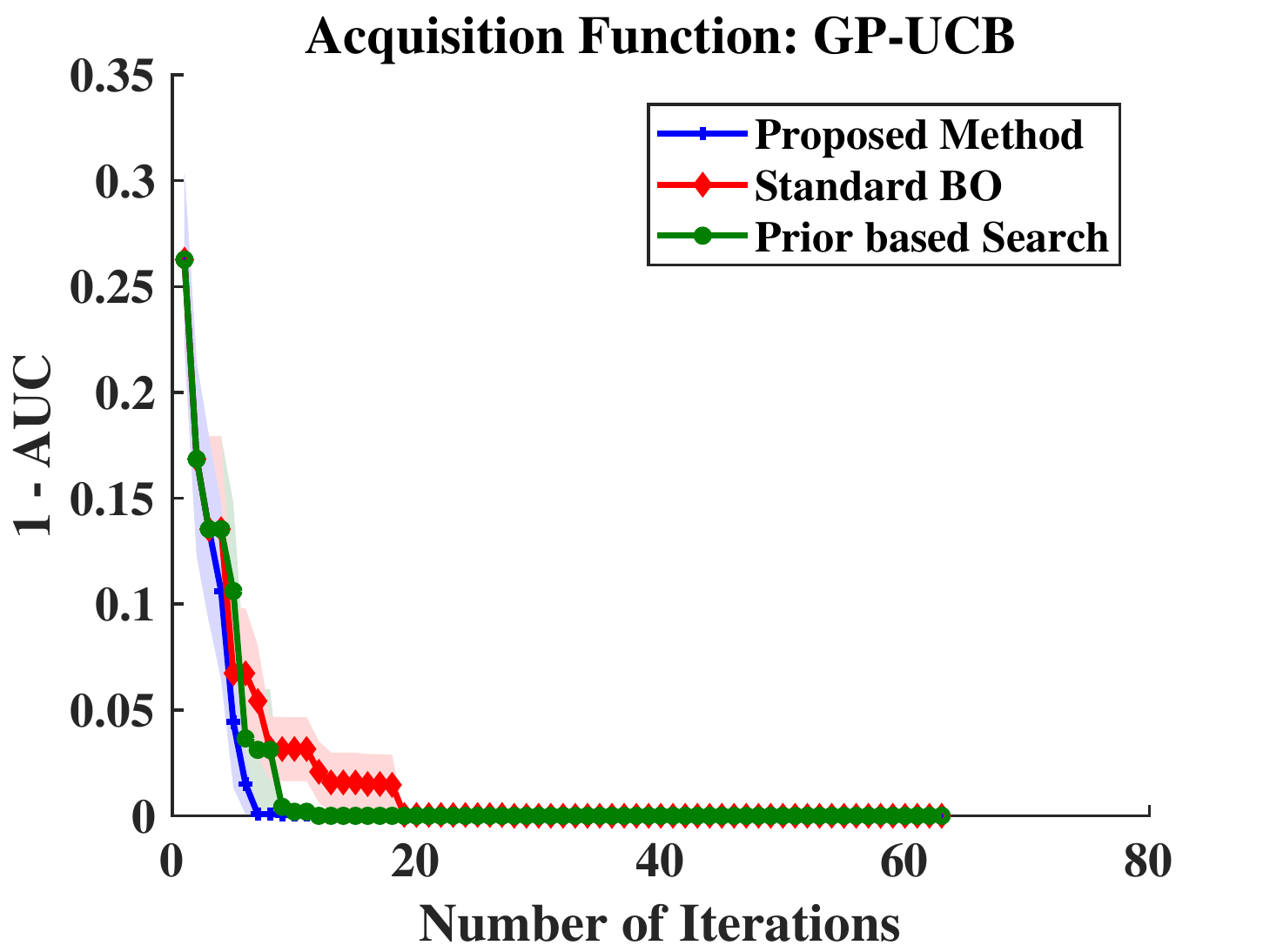}}
\par\end{centering}
\caption{Hyperparameter tuning (SVM with RBF kernel): 1 - AUC vs optimisation
iterations. (a) \& (b) dataset - `german numer'. (c) \& (d) dataset
- `four class'.}
\end{figure*}

\subsubsection{Support Vector Machine (SVM):}

Radial Basis Function (RBF) and sigmoid are two types of kernels used
in SVM. The hyperparameters to tune for RBF kernel, i.e. $k({\bf x},{\bf x}^{'})=\exp\left(-\gamma\parallel{\bf x}-{\bf x}^{'}\parallel^{2}\right)$
are kernel parameter ($\gamma$) and cost parameter ($C$). We set
the range for and $\gamma$ as $\left[10^{-15},10^{3}\right]$ and
for $C$ as $\left[10^{-5},10^{15}\right]$. Similarly for sigmoid
kernel, i.e. $k({\bf x},{\bf x}^{'})=\tanh\left(-\gamma{\bf x}^{T}{\bf x}^{'}+C\right)$,
we have kernel parameter ($\gamma$) and cost parameter ($C$). The
range for $\gamma$ is set as $\left[10^{-15},10^{3}\right]$ and
for $C$ is $\left[10^{-5},10^{15}\right]$ \citep{Van_Hutter_18Hyperparameter}.
For both kernels, and for both hyperparameters, we work in exponent
space. For both kernel types, we perform experiments on two different
binary classification datasets - `german numer' and `four class' from
LibSVM repository \citep{Chang_Lin_11Libsvm}. For both experiments,
we randomly choose $70\%$ of the data and used for training the classifier.
The rest $30\%$ is used for validation. The validation performance
of hyperparameter tuning (measured via AUC) is optimised as a function
of hyperparameter values. For all experiments, we use truncated gamma
distribution as prior with its shape parameter $\alpha$ and inverse
scale parameter $\beta$ estimated from \citep{Van_Hutter_18Hyperparameter}.
The shape and inverse scale parameter values used while performing
the experiments are given in \textbf{Table \ref{tab:Parameter values of truncated gamma distribution}}.
\begin{figure*}
\begin{centering}
\subfloat[\label{fig:Hyperparameter-tuning:exp1 - SVMsigmoid EI - Chapter 5}]{\centering{}\includegraphics[width=0.4\textwidth]{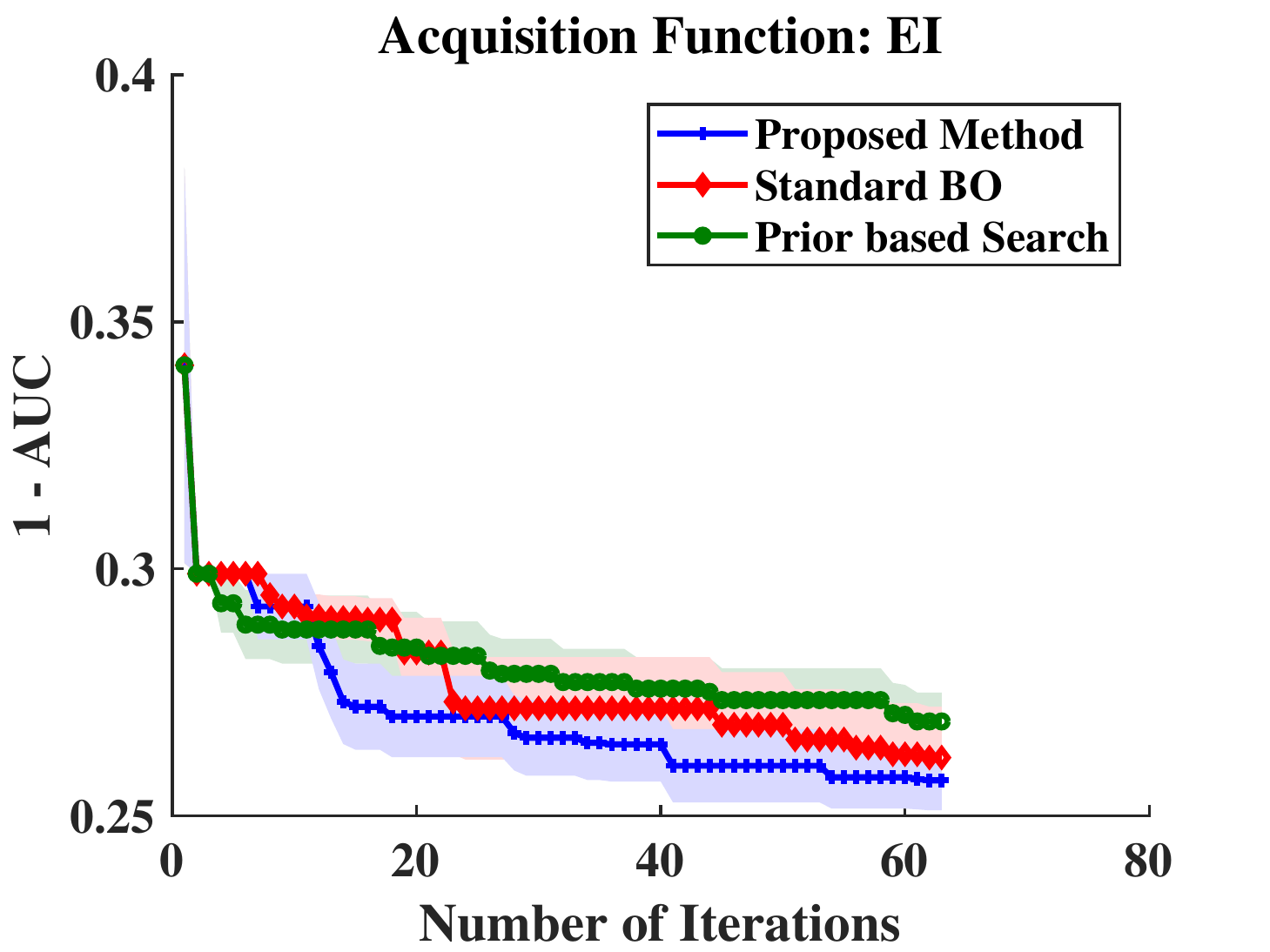}}\subfloat[\label{fig:Hyperparameter-tuning:exp1 - SVMsigmoid GP-UCB - Chapter 5}]{\centering{}\includegraphics[width=0.4\textwidth]{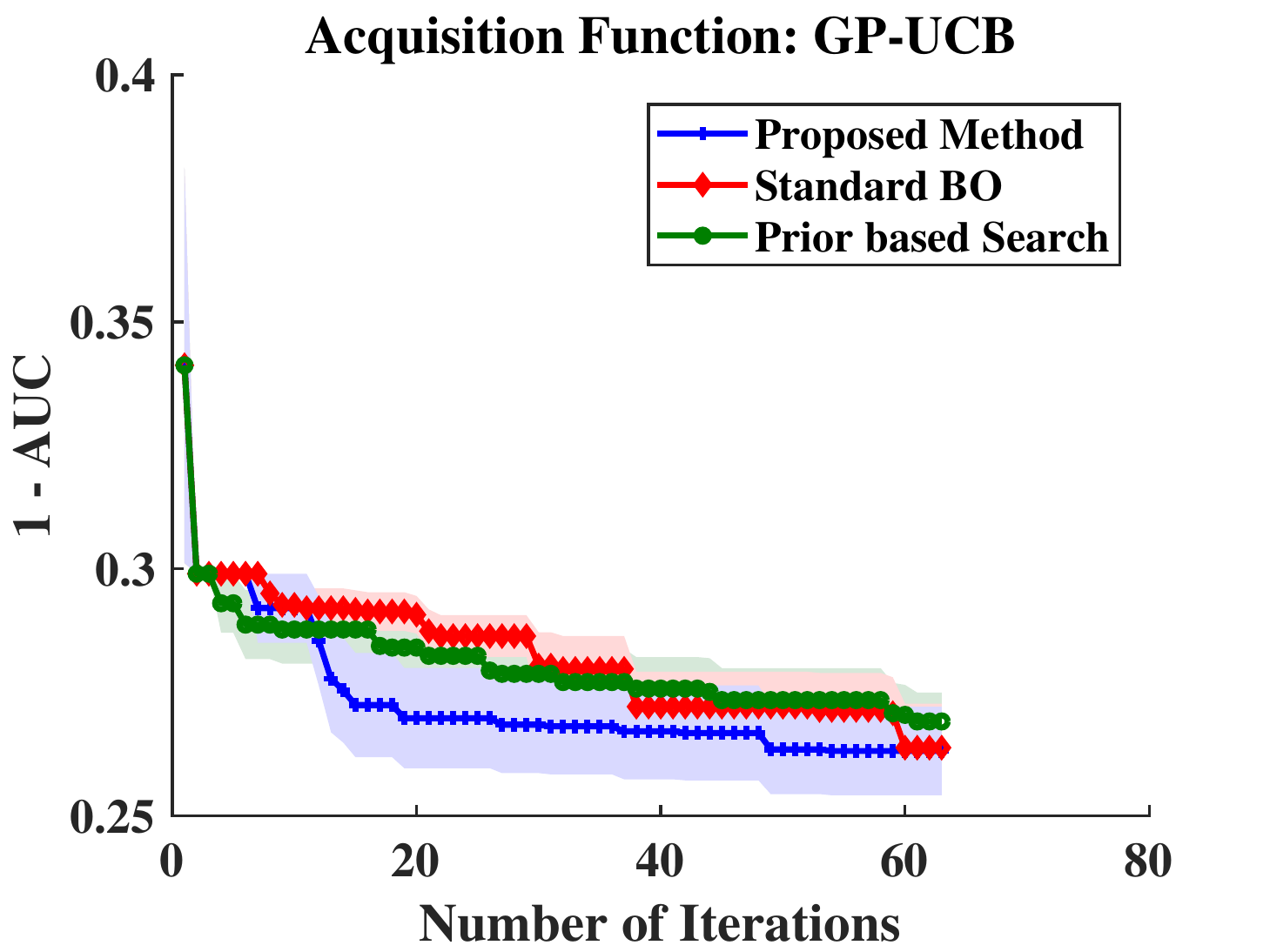}}
\par\end{centering}
\begin{centering}
\subfloat[\label{fig:Hyperparameter-tuning: exp2 SVMsigmoid EI - Chapter 5}]{\centering{}\includegraphics[width=0.4\textwidth]{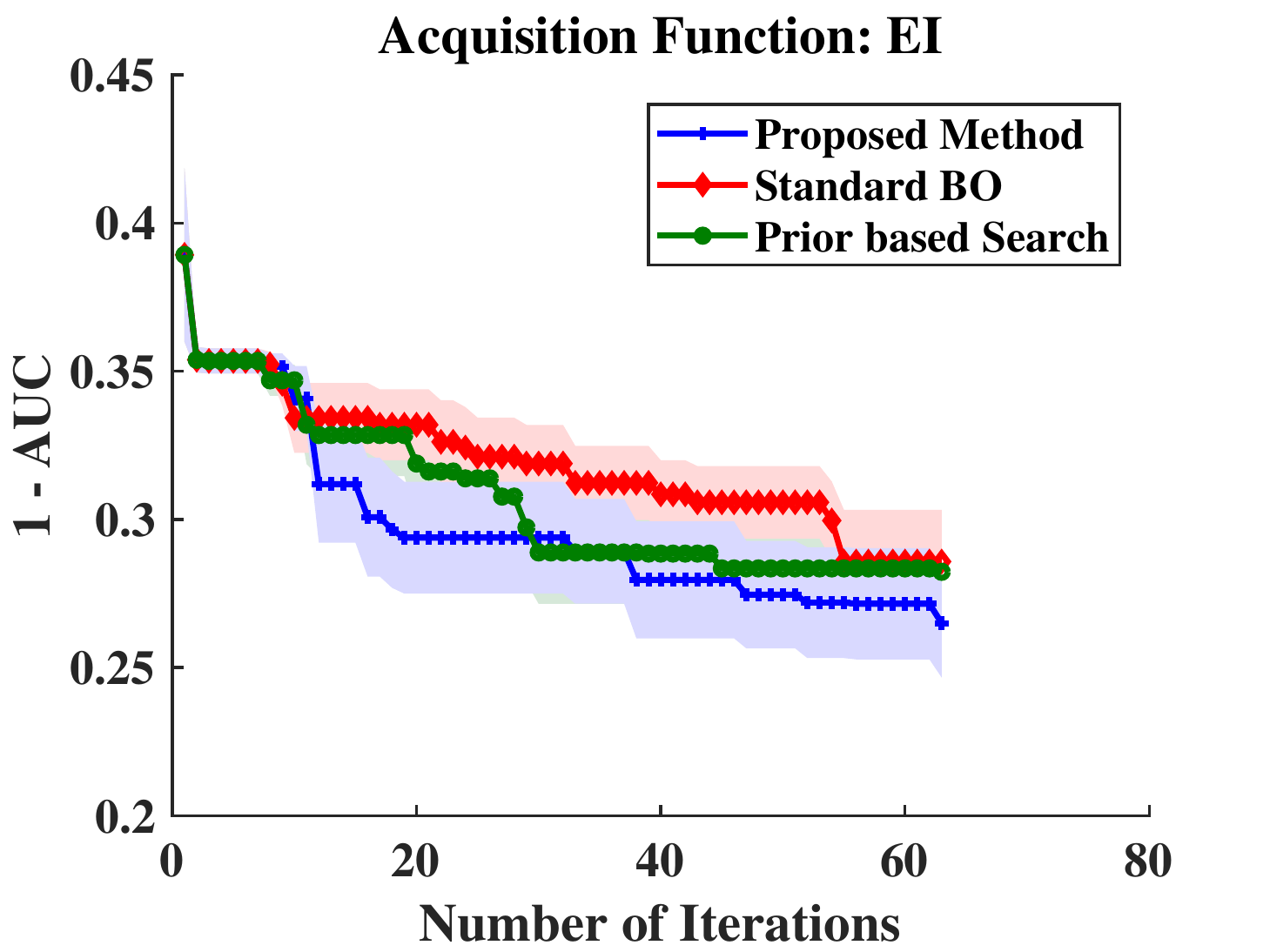}}\subfloat[\label{fig:Hyperparameter-tuning: exp2 SVMsigmoid GP-UCB - Chapter 5}]{\centering{}\includegraphics[width=0.4\textwidth]{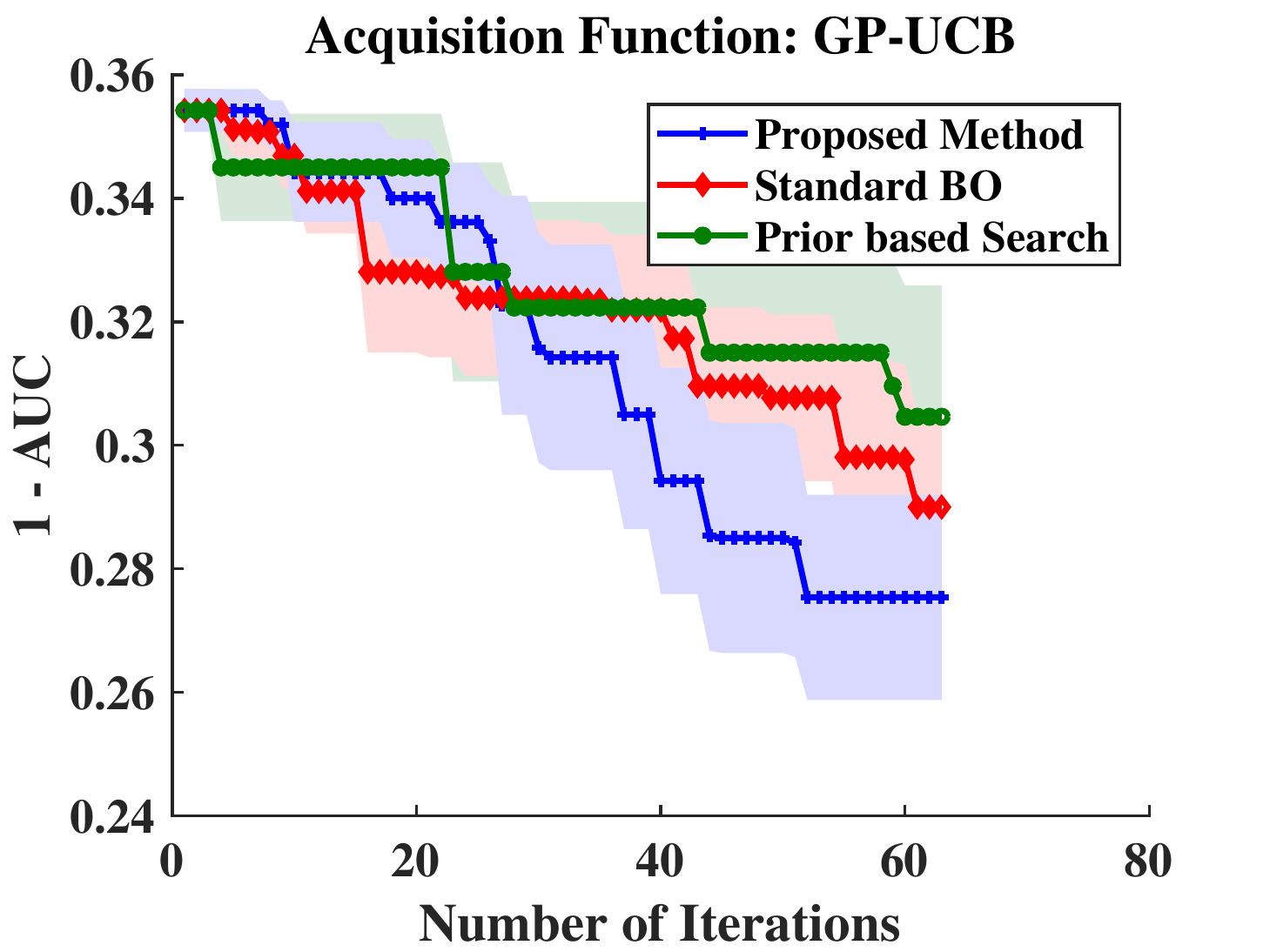}}
\par\end{centering}
\caption{Hyperparameter tuning (SVM with sigmoid kernel): 1 - AUC vs optimisation
iterations. (a) \& (b) dataset - `german numer'. (c) \& (d) dataset
- `four class'.}
\end{figure*}
 Figures \ref{fig:Hyperparameter-tuning:exp1 - SVMRBF EI - Chapter 5}
and \ref{fig:Hyperparameter-tuning:exp1 - SVMRBF GP-UCB - Chapter 5}
shows 1 - AUC performance on a held-out validation set (using both
EI and GP-UCB) while tuning hyperparameters of SVM with RBF kernel
on `german numer' dataset. Using the selected prior, our proposed
method is able to outperform the other two baselines. Figures \ref{fig:Hyperparameter-tuning: exp2 SVMRBF EI - Chapter 5}
and \ref{fig:Hyperparameter-tuning: exp2 SVMRBF GP-UCB - Chapter 5}
shows 1 - AUC performance on a held-out validation set (using both
EI and GP-UCB) while tuning hyperparameters of SVM with RBF kernel
on `four class' dataset. Here also our proposed Bayesian optimisation
method is showing better performance than baselines using the selected
prior. Similarly, the 1 - AUC performance on a held-out validation
set using both EI and GP-UCB acquisition functions while tuning hyperparameters
of SVM with sigmoid kernel on `german numer' dataset and on `four
class' dataset are respectively shown in Figures \ref{fig:Hyperparameter-tuning:exp1 - SVMsigmoid EI - Chapter 5},
\ref{fig:Hyperparameter-tuning:exp1 - SVMsigmoid GP-UCB - Chapter 5},
\ref{fig:Hyperparameter-tuning: exp2 SVMsigmoid EI - Chapter 5} and
\ref{fig:Hyperparameter-tuning: exp2 SVMsigmoid GP-UCB - Chapter 5}.
With the selected prior, our method clearly shows superiority over
other two baselines.

\subsubsection{Random Forest:}

\begin{figure*}
\begin{centering}
\subfloat[\label{fig:Hyperparameter-tuning:exp 1 RF1 EI - Chapter 5}]{\centering{}\includegraphics[width=0.4\textwidth]{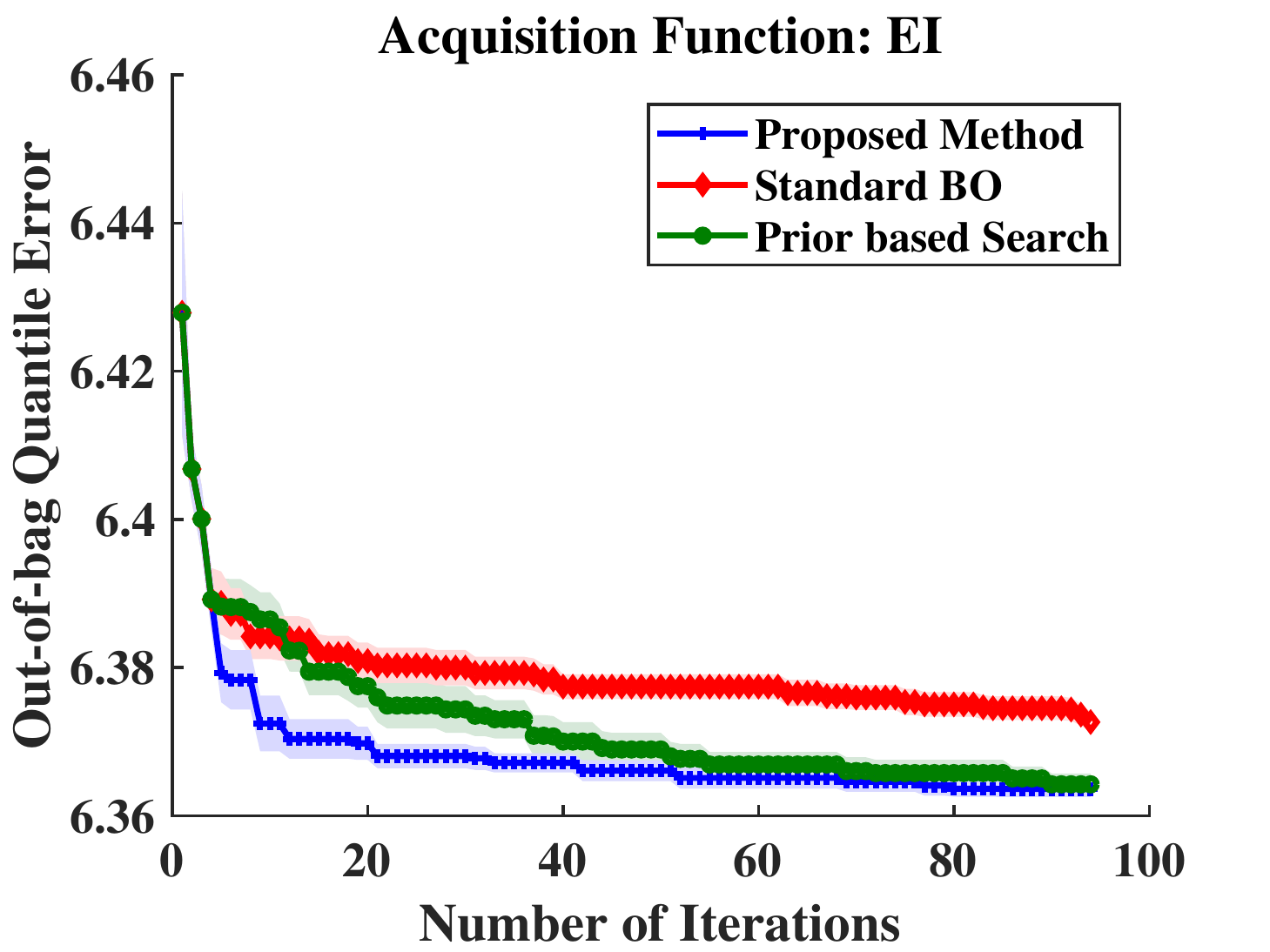}}\subfloat[\label{fig:Hyperparameter-tuning:exp 1 RF1 GP-UCB - Chapter 5}]{\centering{}\includegraphics[width=0.4\textwidth]{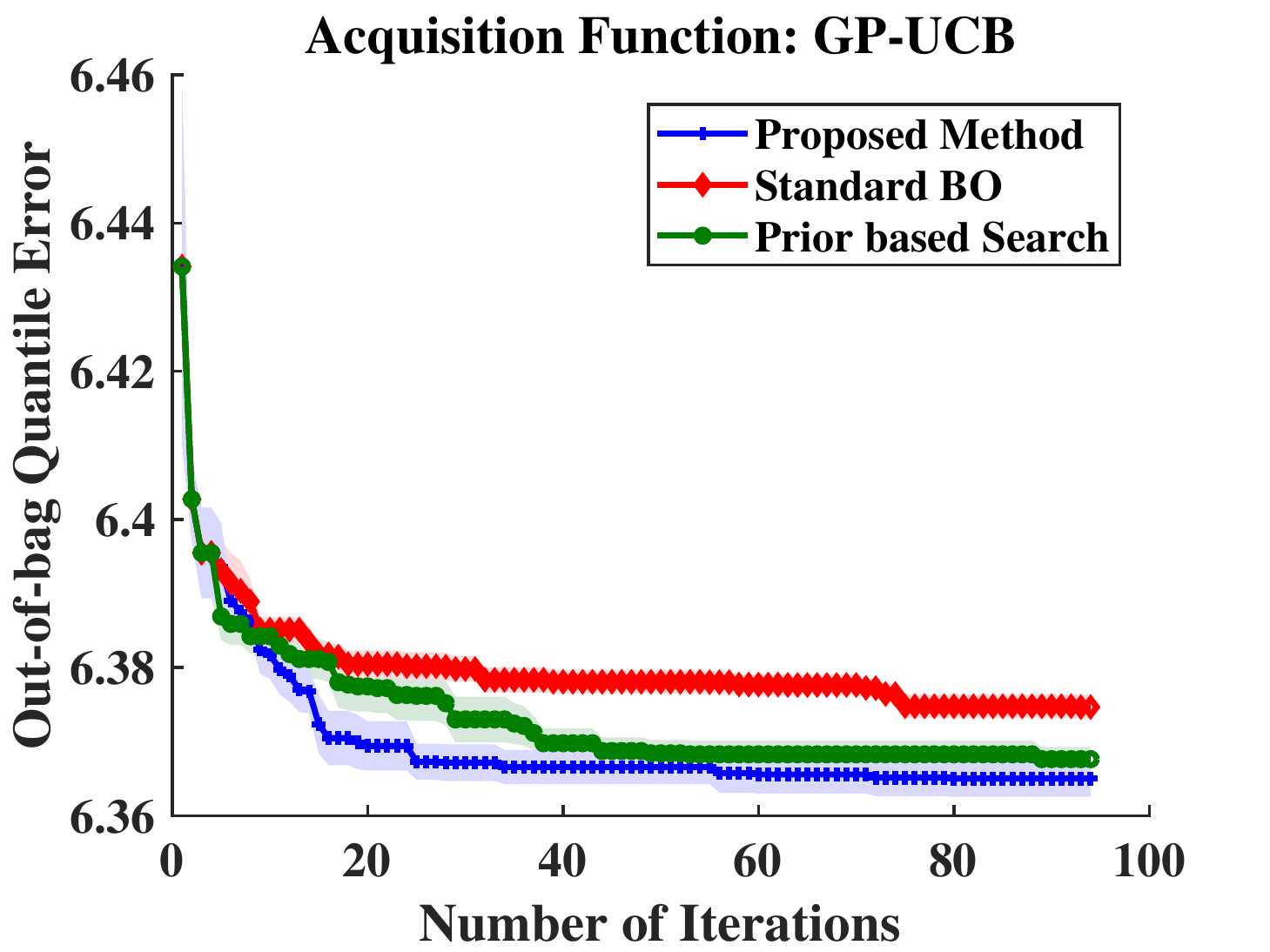}}
\par\end{centering}
\begin{centering}
\subfloat[\label{fig:Hyperparameter-tuning: exp2 RF2 EI - Chapter 5}]{\centering{}\includegraphics[width=0.4\textwidth]{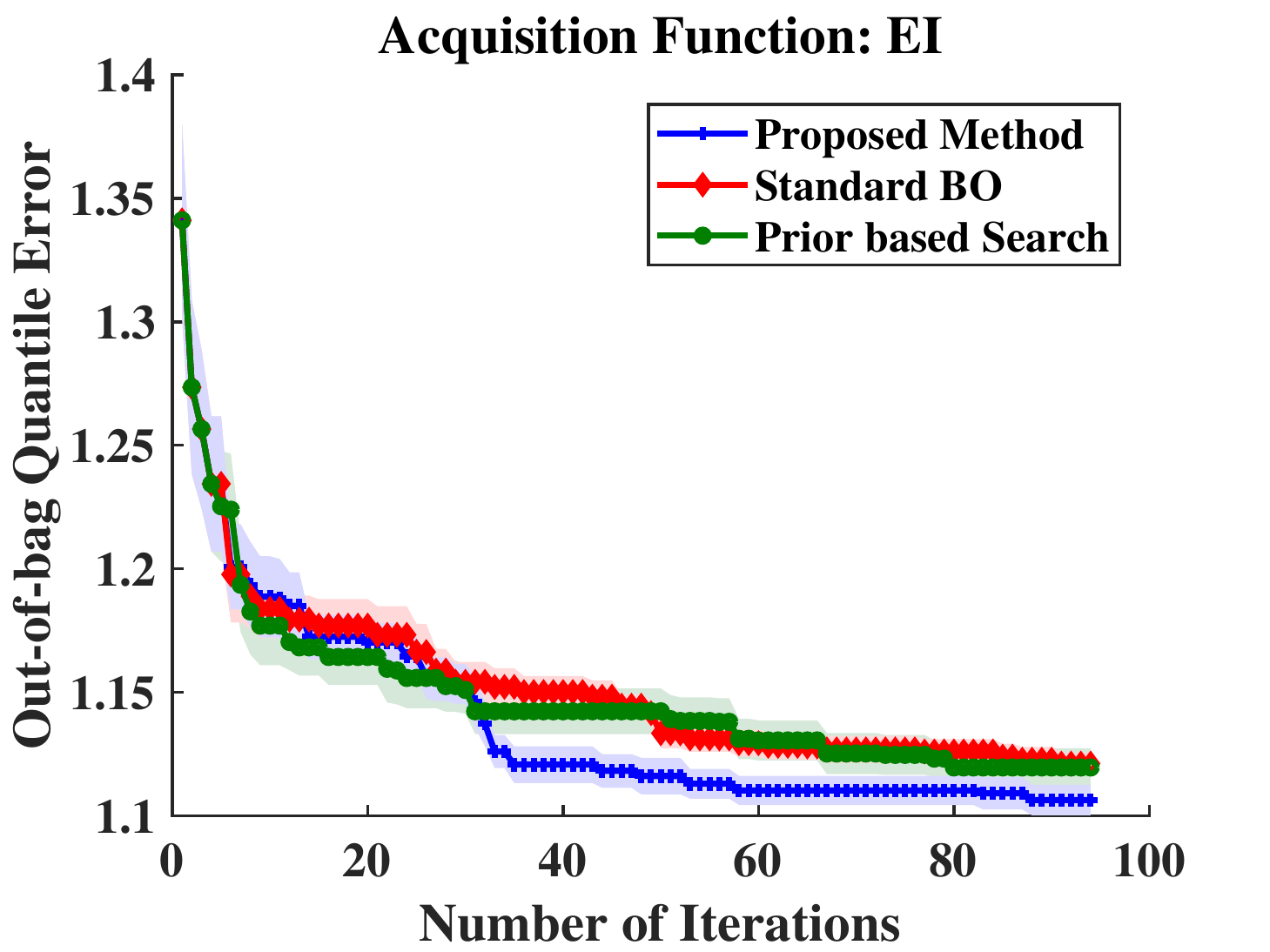}}\subfloat[\label{fig:Hyperparameter-tuning: exp2 RF2 GP-UCB - Chapter 5}]{\centering{}\includegraphics[width=0.4\textwidth]{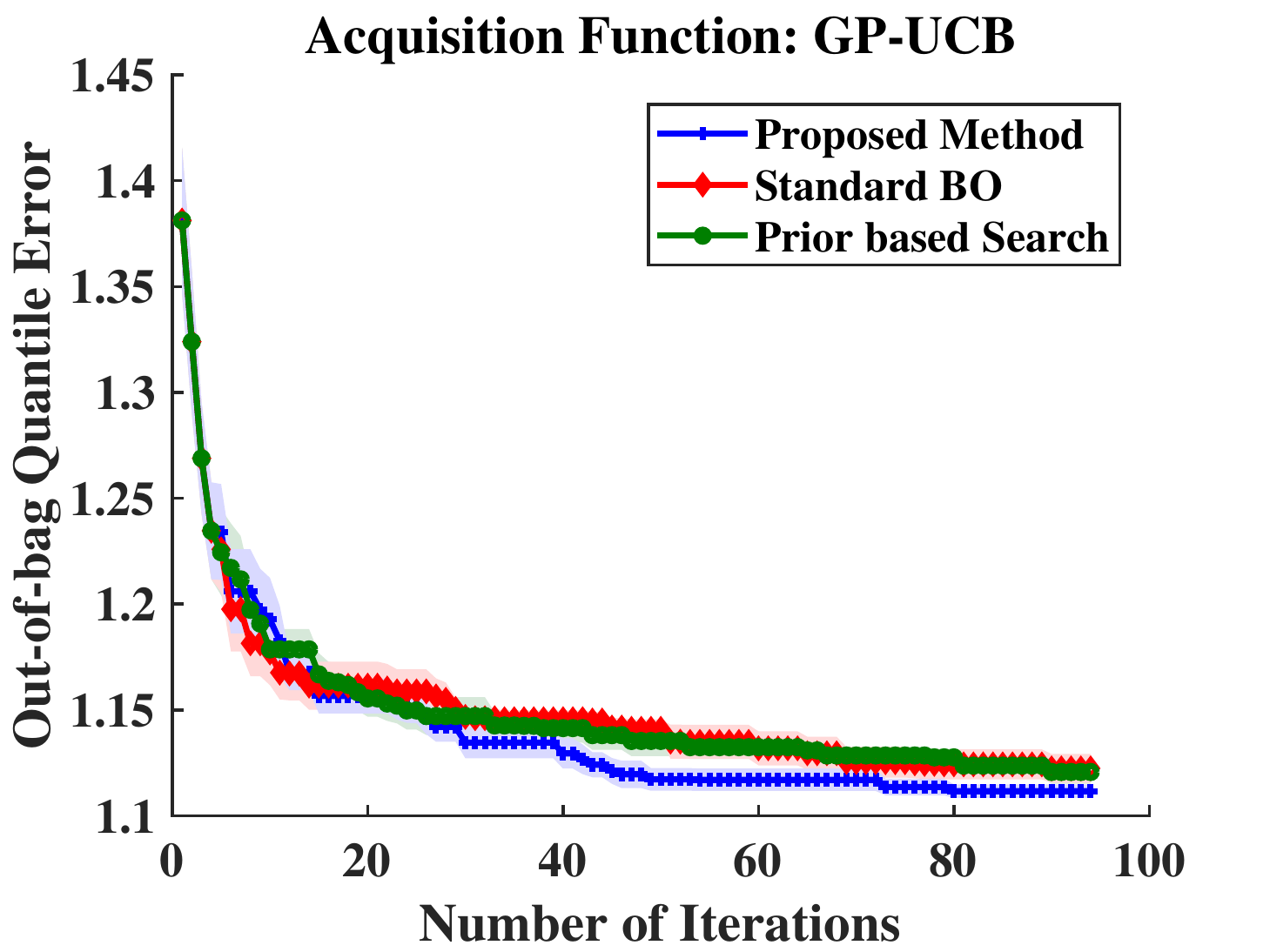}}
\par\end{centering}
\caption{Hyperparameter tuning (Random Forest): Out-of-bag quantile error with
respect to iterations. (a) \& (b) dataset - `forest-fires'. (c) \&
(d) dataset - `carsmall'.}
\end{figure*}
The hyperparameters to tune for random forest are the `minimum samples
per leaf', `maximum number of features sampled per node' and `number
of trees'. The range for `minimum samples per leaf' and `maximum number
of features sampled per node' are set as $\left[1,20\right]$ and
$\left[0.1,0.9\right]$ respectively \citep{Van_Hutter_18Hyperparameter}.
For both of these hyperparameters, we selected truncated gamma distribution
as prior distribution. The range for the `number of trees' is set
as $\left[10,300\right]$. Since there is no clear prior knowledge
available for this hyperparameter, we does not warp the search space
along this dimension. Here the model performance (measured via quantile
error) is optimised as a function of hyperparameters values. We consider
a regression dataset called `forest-fires' - from UCI machine learning
repository \citep{Cortez_Morais_07Adata}. Here the goal is to predict
the burned area of forest fires using some meteorological and other
data. We train the random forest and returns the out-of-bag quantile
error based on the median. For the prior distribution on `minimum
samples per leaf' and `maximum number of features sampled per node',
the shape parameter, $\alpha$ is set as $0.5$ and inverse scale
parameter, $\beta$ as $1$. Figures \ref{fig:Hyperparameter-tuning:exp 1 RF1 EI - Chapter 5}
and \ref{fig:Hyperparameter-tuning:exp 1 RF1 GP-UCB - Chapter 5}
shows out-of-bag quantile error with respect to iterations for our
method and the baselines. With the prior selected, the proposed method
clearly outperforms the baselines. For this dataset, the baseline
Prior based Search performs better than the baseline Standard BO.
We also consider another dataset - `carsmall' from Statistics and
Machine Learning Toolbox. The goal is to predict the median fuel economy
of a car given its various features. For the prior distribution on
`minimum samples per leaf' and `maximum number of features sampled
per node', the shape parameter, $\alpha$ is set as $1$ and inverse
scale parameter, $\beta$ as $0.5$. Out-of-bag quantile error with
respect to iterations is shown in Figure \ref{fig:Hyperparameter-tuning: exp2 RF2 EI - Chapter 5}.
The proposed method shows better performance than other two baselines.

\section{Conclusion\label{sec:conclusion}}

We proposed a novel Bayesian optimisation framework incorporating
prior assumption regarding the optima location of the objective function.
The knowledge about the optima location is represented through a prior
distribution which is then used to warp the search space around the
optima location in such a way that the regions around the optima are
expanded and regions away from the optima are shrunk. We used a well-known
approach called probability integrity transform for warping the search
space. Warped information is incorporated directly into Gaussian process
by redefining its kernel. This enabled the proposed method agnostic
to any acquisition function. We also proved that the redefined (warped)
kernel is a valid Mercer kernel and with this kernel the convergence
of Bayesian optimisation is always guaranteed. Our experiments with
diverse optimisation tasks illustrated the superiority of our proposed
method over standard Bayesian optimisation method and prior based
random search.

\subsection*{Acknowledgment}

This research was partially funded by the Australian Government through
the Australian Research Council (ARC). Prof Venkatesh is the recipient
of an ARC Australian Laureate Fellowship (FL170100006).

\bibliographystyle{elsarticle-num-names}
\bibliography{KBS}

\end{document}